\documentclass[lettersize,journal]{IEEEtran}
\usepackage{ragged2e}
\usepackage{graphicx}
\graphicspath{{PDF/}{/Biography/PDF/}}
\usepackage[T1]{fontenc}
\usepackage[caption=false,font=footnotesize,subrefformat=parens]{subfig}
\usepackage{amsmath}
\usepackage{amssymb}
\usepackage{mdwmath}
\usepackage{bm}
\usepackage{amsthm}
\usepackage{siunitx}
\usepackage{multirow}
\usepackage{booktabs}
\usepackage[hidelinks]{hyperref}
\usepackage[ruled,vlined]{algorithm2e}
\newtheorem{theorem}{Theorem}
\newtheorem{lemma}{Lemma}

\newtheorem{proposition}{Proposition}
\newtheorem{assumption}{Assumption}

\usepackage{amsmath}
\usepackage{pifont}
\usepackage{makecell}
\allowdisplaybreaks[4]
\usepackage{cite}
\usepackage{color}
\begin{document}
\title{{Communication-Efficient} Hybrid Federated Learning for E-health with Horizontal and Vertical Data Partitioning}
\author{Chong~Yu,~\IEEEmembership{Member,~IEEE,}
        Shuaiqi~Shen,~\IEEEmembership{Member,~IEEE,}\\
        Shiqiang~Wang,~\IEEEmembership{Senior Member,~IEEE,}
        Kuan~Zhang,~\IEEEmembership{Member,~IEEE,}
        and~Hai~Zhao
\thanks{Chong~Yu is with the Department of Computer Science, University of Cincinnati, Cincinnati, OH 45221, USA. E-mail: {yuc5@ucmail.uc.edu}}
\thanks{Shuaiqi~Shen is with the Department of Electrical Engineering, University of Wisconsin Milwaukee, Milwaukee, WI 53211, USA. E-mail: {shen8@uwm.edu}}
\thanks{Shiqiang~Wang is with the IBM T. J. Watson Research Center, Yorktown Heights, NY 10598, USA. E-mail: {wangshiq@us.ibm.com}}
\thanks{Kuan~Zhang is with the Department of Electrical and Computer Engineering, University of Nebraska-Lincoln, Lincoln, NE 68588, USA. E-mail: {kuan.zhang@unl.edu}}
\thanks{Hai~Zhao is with the Department of Computer Science and Engineering, Northeastern University, Shenyang 110819, China. E-mail: {zhaoh@mail.neu.edu.cn}}}

\maketitle
\begin{abstract}
E-health allows smart devices and medical institutions to collaboratively collect patients' data, which is trained by Artificial Intelligence (AI) technologies to help doctors make diagnosis. By allowing multiple devices to train models collaboratively, federated learning is a promising solution to address the communication and privacy issues in e-health. However, applying federated learning in e-health faces many challenges. First, medical data is both horizontally and vertically partitioned. Since single Horizontal Federated Learning (HFL) or Vertical Federated Learning (VFL) techniques cannot deal with both types of data partitioning, directly applying them may consume excessive communication cost due to transmitting a part of raw data when requiring high modeling accuracy. Second, a naive combination of HFL and VFL has limitations including low training efficiency, unsound convergence analysis, and lack of parameter tuning strategies. In this paper, we provide a thorough study on an effective integration of HFL and VFL, to achieve communication efficiency and overcome the above limitations when data is both horizontally and vertically partitioned. Specifically, we propose a hybrid federated learning framework with one intermediate result exchange and two aggregation phases. Based on this framework, we develop a \underline{H}ybrid \underline{S}tochastic \underline{G}radient \underline{D}escent (HSGD) algorithm to train models. Then, we theoretically analyze the convergence upper bound of the proposed algorithm. Using the convergence results, we design adaptive strategies to adjust the training parameters and shrink the size of transmitted data. Experimental results validate that the proposed HSGD algorithm can achieve the desired accuracy while reducing communication cost, and they also verify the effectiveness of the adaptive strategies.
\end{abstract}


\section{Introduction}
\label{sec:introduction}
\IEEEPARstart{E}{lectronic} healthcare (e-health) connects smart devices and healthcare providers via Internet-of-Things (IoT) technologies to offer intelligent health services. It can solve several problems in traditional medical systems, such as insufficient professional workforce, low efficiency of diagnosis, and heavy burden of population aging. Due to the ambitious prospects, {e-health} has become increasingly popular. Statistics show that about 87 million American residents experienced e-health services monthly in 2020, and the number is projected to steadily increase in the future~\cite{Mobile_Health_Statistics}. Meanwhile, the global e-health market reached a value of USD 62.4 billion in 2021 and is expected to grow about 12 times by 2028~\cite{Telehealth_Market}.

\begin{figure}[!t]
	\centering
	\includegraphics[width=\columnwidth]{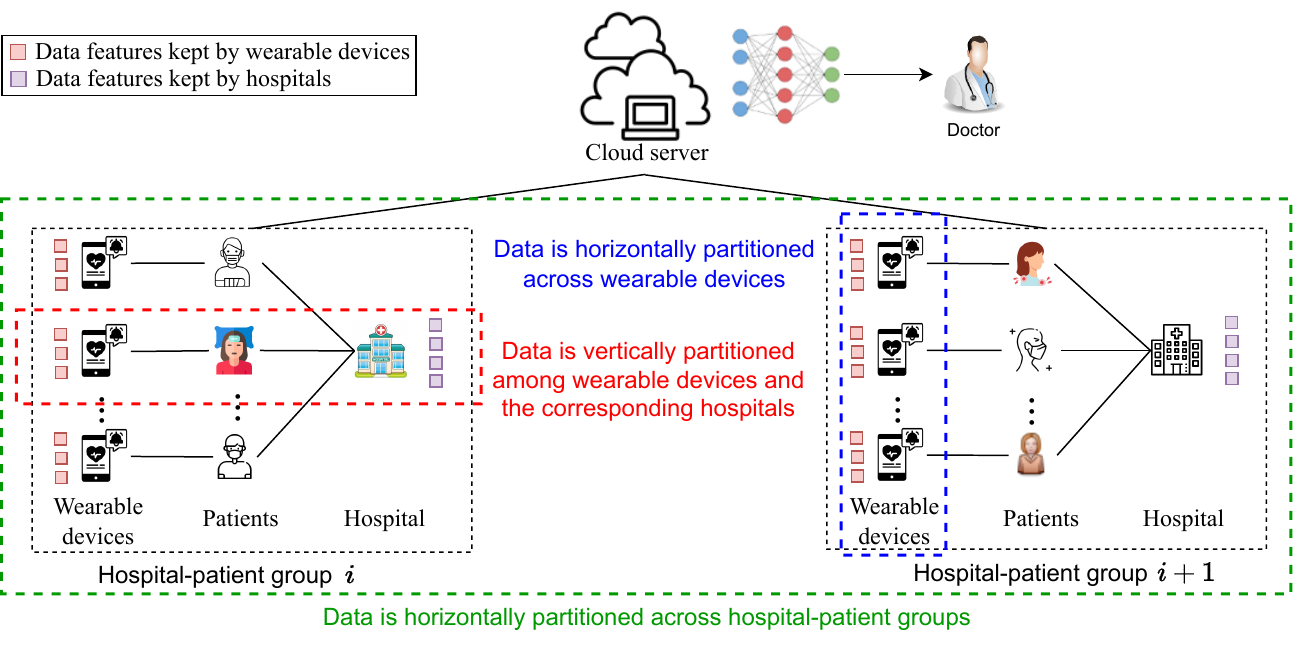}
	\captionsetup{font={scriptsize}}
    \caption{E-health architecture: e-health has a three-tier horizontal-vertical-horizontal data distribution structure.}
	\label{E_health_network_architecture}
\end{figure}

An e-health system consists of wearable devices, edge nodes, hospitals, and a cloud server, where components are interconnected via communication networks, as illustrated in Fig.~\ref{E_health_network_architecture}. Each patient is equipped with a wearable device and is only affiliated with a hospital, and each hospital caters to a large number of patients. Different from other mobile computing applications, \text{e-health} has unique data characteristics~\cite{Madhura2020e_health_FL}. First, patients can utilize wearable devices, such as smartphones and smartwatches, to collect daily health data, including body temperature, blood pressure, and heart rate~\cite{bayoumy2021smart}. Different wearable devices gather the same health indicators - body temperature, blood pressure, and heart rate for distinct patients. This indicates that the dataset over various wearable devices shares identical data features but differs in samples, i.e., data is horizontally partitioned across wearable devices (as shown in the blue dash rectangle in Fig.~1). Second, a patient not only use wearable devices to self-monitor~\cite{appelboom2014smart}, e.g., recording daily health data, but also visit hospitals for comprehensive physical examinations~\cite{oboler2002public}, e.g., collecting lab results. For each patient, its entire health record consists of daily health data and lab results, where these two kinds of data are disjoint. In other words, a wearable device and the corresponding hospital keep different data features for one patient, i.e., data is vertically partitioned among wearable devices and the corresponding hospitals (as shown in the red dashed rectangle in Fig.~1). Third, patients who visit the same hospital form a hospital-patient group. Two hospital-patient groups in different locations may have no overlap in patients, yet each patient within a group possesses their entire health data, demonstrating that data is horizontally partitioned across hospital-patient groups (as shown in the green dashed rectangle in Fig.~1)~\cite{Madhura2020e_health_FL}. Therefore, e-health has a three-tier horizontal-vertical-horizontal data distribution structure.

Although e-health has promising potential, it faces many challenges caused by its data characteristics. An entire health record of a patient consists of two parts, where simple vital signs such as blood sugar values are kept by the patient's wearable device, and the lab results such as X-Ray images are held by the corresponding hospital. Since these two parts of data are usually complementary, both of them are required to assess the patient's condition. However, transmitting the raw data stored in different locations causes two issues. On one hand, patients with chronic diseases, such as cardiovascular diseases, need to wear smart sensors for a long time to monitor their body conditions, and a large volume of vital sign data is generated continuously in real time. In this case, frequently transmitting patients’ raw data increases overhead. On the other hand, the raw data may include the patient's age, gender, and even identity. Transmission of raw data may leak patients' privacy. To reduce communication overhead and protect privacy, e-health requires \textit{federated learning}, which allows multiple devices to collaboratively train models without transmitting raw data.

According to data partitioning, existing federated learning techniques can be classified into two types. One type is  Horizontal Federated Learning (HFL), which enables each device to independently learn a local model based on its own data and allows the cloud server to generate a global model by aggregating local models~\cite{Globecom-Zhu,HFL1}. In e-health, partial data features are held by wearable devices, and the rest data features are stored in hospitals. The standalone HFL can only process the incomplete dataset on either wearable devices or hospitals. If HFL wants to train an accurate model in e-health, part of the raw data needs to be transmitted between wearable devices and hospitals to include all features. This inevitably increases the communication cost. Another type is Vertical Federated Learning (VFL), which coordinates multiple devices with complementary data to learn a model by communicating intermediate results~\cite{VFL,VFL1}. In e-health, patients who visit the same hospital form a group. Multiple sub-datasets with identical data features from different groups compose an entire dataset. VFL can only handle a single sub-dataset. If VFL wants to learn an accurate model, additional raw data transmission is required to combine all samples, resulting in a heavy communication burden and patients' information disclosure. In addition, if keeping original data locally in the above two cases, the communication cost is saved; but the model is inaccurate because of ignoring a portion of the data. A straightforward solution to guarantee modeling accuracy and reduce communication cost is combining HFL and VFL.

However, simply combining HFL and VFL is challenging in terms of convergence analysis, training efficiency, and parameter tuning strategies. First, HFL and VFL interact upon integration, making theoretical proof of convergence hard when combining the two. Convergence results can offer insights for optimizing the training process, such as determining the optimal communication frequency. Lack of such results poses a challenge in achieving optimal training performance. Second, when integrating HFL and VFL, the intermediate results used in VFL is obtained from HFL. HFL generally conducts multiple local updates between two communications, which may result in a lag in the intermediate results. Due to this lag, modeling efficiency is lower compared to global multi-class classification methods~\cite{Ahn2020CIFAR10}. Third, training efficiency is affected by several major configurations, such as the interval between two global aggregations, the frequency of local communications, and the learning rate. Adjusting these settings to optimize training performance is necessary but difficult because the complex effect of parameters on modeling performance is unknown.

In this paper, we propose a communication-efficient \textit{hybrid federated learning} approach that addresses the above challenges. The proposed hybrid federated learning considers a three-tier horizontal-vertical-horizontal structure and includes local aggregation, intermediate result exchange, and global aggregation phases. Our proposed approach can efficiently deal with horizontally and vertically partitioned medical data while reducing communication overhead. Furthermore, compared with existing joint federated learning techniques~\cite{Yu2022FL,Das2021CSFLjournal} with a two-tier vertical-horizontal structure, our proposed approach introduces an additional aggregation process to improve training efficiency. In addition, our proposed hybrid federated learning has a higher degree of freedom than two-tier joint federated learning to tune parameters for achieving high accuracy. In summary, our main contributions are as follows.

\begin{itemize}
	\item We propose a hybrid federated learning framework for \text{e-health} consisting of a horizontal-vertical-horizontal structure. To the best of our knowledge, this is the first hybrid federated learning framework with such a three-tier structure.
	Based on this framework, we develop a communication efficient \underline{H}ybrid \underline{S}tochastic \underline{G}radient \underline{D}escent~(HSGD) algorithm to train models when medical data is both horizontally and vertically partitioned. The proposed algorithm can save communication cost to achieve a desirable training requirement.
 
	\item We formulate the training process defined by the developed HSGD algorithm. Considering a minimal set of assumptions on loss functions, we theoretically analyze the global gradient upper bound. According to this upper bound, we derive the conditions that guarantee the convergence of the HSGD algorithm. The insights from these theoretical analyses shed light on how to improve the communication efficiency of hybrid federated learning.
 
	\item Analyzing the convergence results, we deduce how global and local aggregation intervals ($P$ and $Q$) and the learning rate $\eta$ affect the training results. Then we design three insightful adaptive strategies for guiding the setting of $Q$, $P$, and $\eta$ to reduce communication cost and improve accuracy. This differs our work from existing combinations of HFL and VFL as they did not provide any adaptive strategies.
 
	\item In addition to the above theoretical breakthrough, we conduct extensive experiments to validate the proposed HSGD algorithm and adaptive strategies.  Compared with four baselines, the proposed HSGD algorithm can achieve the desired accuracy, and save training time and communication overhead. In addition, the experimental results verify the effectiveness of adaptive strategies for lessening the communication burden.
\end{itemize}

The remainder of this paper is organized as follows. Section~\ref{sec:Related_Works} reviews related works. Section~\ref{sec:Hybrid_Federated_Learning} introduces the system model, hybrid federated learning framework, and problem formulation. Section~\ref{sec:Hybrid_Stochastic_Gradient_Descent_Algorithm} presents the proposed HSGD algorithm and Section~\ref{sec:Convergence_analysis} proves the convergence. Then, Section~\ref{sec:Strategies_for_Minimizing_Global_Loss_Function} presents the adaptive strategies. The experimental results are shown in Section~\ref{sec:Experimental_Evaluation} and conclusion is drawn in Section~\ref{sec:Conclusion}.

\section{Related Works}
\label{sec:Related_Works}
Federated learning, as a promising distributed computing paradigm, has been applied in various fields to achieve privacy-preserving learning. The existing federated learning techniques include HFL, VFL, and combinations of HFL and VFL, to deal with different data partitioning scenarios.

\textbf{HFL and its variations.} HFL can train models for the scenarios where data is horizontally partitioned. McMahan et al.~\cite{McMahan2017} first proposed the Federated Averaging (FedAvg) algorithm for HFL. In FedAvg, devices locally perform a fixed number of gradient descent updates, and the cloud server periodically aggregates local models. Considering resource constraints, many methods~\cite{Sun2021AFL,Nori2021FFL,Ma2021FL,Sattler2020Communication_Efficient_FL,Xu2021FL,Shen2021FL,Yang2020Communication_Efficient_FL, Jiang2023, Xu2022, Shah2023, Han2023Practical} based on FedAvg were proposed to achieve communication efficiency and fast convergence in the horizontal data partitioning scenario. For example, Jiang et al~\cite{Jiang2023} adjusted the model size during HFL to minimize overall training time. Xu et al.~\cite{Xu2022} proposed a ternary quantization algorithm to reduce communication cost in HFL. Additionally, Shah et al.~\cite{Shah2023} achieved efficient communication in HFL by applying compression techniques for both the server model and client models. Furthermore, practical and robust HFL was investigated in~\cite{Han2023Practical} to improve security and privacy. \textit{However, HFL and its variations face challenges when applied to e-health because they only consider horizontally partitioned data while cannot process vertically partitioned data. To keep modeling accuracy, such techniques need to transmit raw data for combining the data stored on devices and hospitals, but this requires transmitting a lot of raw data.}

\textbf{VFL and its variations.} In~\cite{VFL}, the concept of federated learning was extended to VFL when data is vertically partitioned. To coordinate multiple devices for learning and protecting privacy, many VFL algorithms were proposed~\cite{Zhu2021VFL,Das2021VFL,Gu2022Privacy,Zhang2022VFL,xu2021fedv,liu2022vertical,khan2022communication,pouriyeh2022secure}. In these algorithms, devices first perform forward propagation based on their sub-models and local data to obtain intermediate results. Then, devices with complementary data features for the same sample space exchange intermediate results and conduct backward propagation to update the sub-models. To further enhance the robustness of VFL, Castiglia et al.~\cite{Castiglia2023} proposed a flexible VFL to support heterogeneous and time-varying parties. Additionally, Gu et al.~\cite{Gu2022Privacy} designed an asynchronous federated stochastic gradient descent algorithm to improve communication efficiency while keeping data privacy in VFL. \textit{Although VFL and its variations show significant advantages, they are not efficient when handling horizontally and vertically partitioned medical data since these methods only focus on vertical data partitioning scenarios. The training accuracy of such techniques can be improved by transmitting raw data to integrate horizontally partitioned data, but this also increases the communication burden.}

\textbf{Combinations of HFL and VFL.} To enhance the scalability of federated learning, some existing works combine HFL and VFL for the scenarios where horizontally and vertically partitioned data co-exist. Yu et al.~\cite{Yu2022FL} proposed a Joint Federated Learning (JFL) algorithm. In JFL, a local model is composed of an edge model, a server model, and a combined model. VFL is conducted between edge nodes and the server to update local models, and HFL is performed among edge nodes to generate the global model. Nevertheless, this work did not provide any theoretical analysis results to prove the convergence of JFL. Recently, Das et al.~\cite{Das2021CSFLjournal} proposed a Tiered Decentralized Coordinate Descent (TDCD) algorithm, where devices with the same feature subsets in each silo conduct HFL and hubs of silos with the same sample spaces perform VFL. However, they did not study how to appropriately tune the learning parameters, resulting in inefficient training, e.g., 40\% to 60\% test accuracy on CIFAR-10~\cite{Das2021CSFLjournal}. Such low accuracy cannot satisfy practical requirements, especially in e-health which is extremely sensitive to accuracy. In addition, these works only considered two-tier structures. They did not consider a three-tier horizontal-vertical-horizontal structure. Although a similar framework was described in our preliminary work~\cite{Yu2022ICCFL}, it has several limitations: i) used a stronger assumption on the loss function, ii) lacked detailed convergence proofs, iii) contained a non-zero convergence bound, iv) did not give any parameter tune schemes, and v) only focused on accuracy without considering communication cost. In contrast, we address all aforementioned issues in this paper, achieving a minimal set of assumptions, solid convergence proofs, zero convergence bound, adaptive strategies on parameters, and the trade-off between training accuracy and communication cost, which are validated by extensive experiments.

In summary, existing federated learning works have the following limitations when applied to e-health. HFL or VFL alone transfers a massive volume of data to train medical data because the data is both horizontally and vertically partitioned. Although many works attempt to combine HFL and VFL to cope with the complex data partitioning while alleviating the communication burden, most of these works lack theoretical support. Furthermore, training efficiency is affected by several important parameters, such as the interval between two global aggregations, the frequency of intermediate result communications, and the learning rate. Despite some works utilizing a minimal set of assumptions to analyze convergence, they cannot achieve high training efficiency due to lacking modeling parameter tuning schemes. Therefore, we should effectively integrate HFL and VFL to balance modeling accuracy and communication cost, while providing sound convergence analysis and parameter tuning strategies.

\section{Hybrid Federated Learning}
\label{sec:Hybrid_Federated_Learning}

\subsection{E-health System Model}
\label{sec:E-health_network_system_model}
In e-health, a patient is equipped with one wearable device and belongs to only one hospital, and a hospital serves a large number of patients. A hospital and its patients constitute a hospital-patient group. Multiple hospital-patient groups and a remote server form an \text{e-health} system. We define the number of hospital-patient groups as $M$ and the number of patients in each hospital-patient group as $K_m, \forall m \in\{1, ..., M\}$. Since a wearable device merely collects information for one patient, we assume that each wearable device only has one sample.

\begin{table}[!t]
\ifCLASSOPTIONtwocolumn
\renewcommand{\arraystretch}{1}
\fi
\setlength{\extrarowheight}{2pt}
\centering
\caption{Summary of main notations.\label{Table1}}
\begin{tabular}{ l | l }
	\toprule
	\textbf{\text{Notations}}                                      &\textbf{\text{Description}}  \\
	\hline
	\hline
	\multirow{2}*{\shortstack{$\mathbf X_1^{(m,n)}$}}              & Feature subset of sample $n$ kept \\ &by the hospital in hospital-patient group $m$ \\
	\hline
	\multirow{2}*{\shortstack{$\mathbf X_2^{(m,n)}$}}              & Feature subset of sample $n$ held \\ & by wearable device $n$ in hospital-patient group $m$ \\
	\hline
	$\tilde{\bm\theta}$                                            & Global model parameter\\
	\hline
 	$\tilde{\bm\theta}_0$                                          & Global combined model parameter\\
	\hline
 	$\tilde{\bm\theta}_1$                                          & Global hospital side model parameter\\
	\hline
	$\tilde{\bm\theta}_2$                                          & Global device side model parameter\\ 
	\hline
	${\bm\theta}_{m}$                                              & Local model parameter of hospital-patient group $m$ \\ 
	\hline
	\multirow{2}*{\shortstack{${\bm\theta}_{0,m}$}}                & Local combined model parameter of \\ &hospital-patient group $m$\\
	\hline
	\multirow{2}*{\shortstack{${\bm\theta}_{1,m}$}}                & Local hospital side model parameter of \\ &hospital-patient group $m$\\
	\hline
	\multirow{2}*{\shortstack{${\bm\theta}_{2,m}$}}                & Local device side model parameter of \\ &hospital-patient group $m$\\
	\hline
	\multirow{2}*{\shortstack{${\bm\theta}_{2,m,n}$}}              & Model parameter on the $n$-th device of \\ &hospital-patient group $m$\\
	\hline
	$\mathcal{A}_m$                                                & Device subset of hospital-patient group $m$\\
	\hline
	$F(\tilde{\bm\theta})$                                         & Global loss function\\
	\hline
	$F_{m}(\tilde{\bm\theta})$                                     & Local loss function of hospital-patient group $m$\\
	\hline
	$P$                                                            & Global aggregation interval\\
	\hline
	\multirow{2}*{\shortstack{$Q$}}                                & Local aggregation interval and \\ &intermediate result exchange interval\\
	\hline
	\multirow{2}*{\shortstack{${\bm \zeta}_{1,m,n}^t$}}            & Embedding vector of sample $n$ \\ &on the hospital in hospital-patient group $m$\\
	\hline
	\multirow{2}*{\shortstack{${\bm \zeta}_{2,m,n}^t$}}            & Embedding vector of sample $n$ \\ &on the wearable device $n$ in hospital-patient group $m$\\
	\hline
    ${\xi}_m$                                                      &Mini-btach of group $m$\\
	\hline
    \multirow{2}*{\shortstack{$\bm g_{i,m}$}}                      &Local stochastic partial derivative of the loss function\\&with respect to ${\bm\theta}_{i,m}^t, \forall i=0,1$ in group $m$\\
    \hline
    \multirow{2}*{\shortstack{$\bm g_{2,m,n}$}}                    &Local stochastic partial derivative of the loss function\\&with respect to ${\bm\theta}_{2,m,n}^t$ in group $m$\\
	\hline
    $\eta$                                                         & learning rate\\
	\bottomrule
\end{tabular}
\label{table1}
\end{table}

For an \text{e-health} system, the entire dataset $D$, which comprises $K$ samples, is horizontally partitioned into $M$ sub-datasets
following the non-independent-and-identically-distributed (non-iid) data distribution~\cite{adnan2022federated}, denoted as $D=\cup_{m=1}^M D_m$. Each sub-dataset ${D_m=\{\mathbf X^{(m,n)}, y^{(m,n)}\}^{K_m}_{n=1}}$ held by hospital-patient group $m$ has $K_m$ samples, where $\mathbf X^{(m,n)}$ represents the feature vector of the $n$-th sample in hospital-patient group $m$, and $y^{(m,n)}$ denotes the corresponding target prediction. 
For each sample in sub-dataset ${D_m}$, partial data features are collected by hospitals based on patient visits and lab results, and the rest of data features are generated by wearable devices. To capture this characteristic, each feature vector $\mathbf X^{(m,n)}$ is further vertically partitioned into ${\mathbf X^{(m,n)}:= [(\mathbf X_1^{(m,n)})^\mathrm{T}, (\mathbf X_2^{(m,n)})^\mathrm{T}]^\mathrm{T}}$, where $\mathbf X_1^{(m,n)}$ denotes the feature subset saved by the hospital in hospital-patient group $m$, and~$\mathbf X_2^{(m,n)}$ represents the corresponding feature subset stored on wearable device $n$ in hospital-patient group~$m$. The hospital in group~$m$ preserves $\{\mathbf X_1^{(m,n)}\}_{n=1}^{K_m}$ because each hospital collects data for all its visiting patients, while wearable device $n$ in group $m$ only maintains $\mathbf X_2^{(m,n)}$ since each device only gathers information for one patient.  Note that both the hospital and wearable device $n$ in group $m$ can access the desired prediction~$y^{(m,n)}$. The main notations used in this paper are listed in {Table~\ref{table1}}.

\subsection{Hybrid Federated Learning Framework}
\label{sec:Hybrid_federated_learning_framework}
To apply federated learning in e-health, we propose a hybrid federated learning framework, as shown in Fig.~\ref{Hybrid_federated_learning_framework}. The framework consists of one intermediate result exchange and two aggregation phases.

\textbf{Intermediate result exchange phase.} Two reasons motivate the intermediate result exchange. First, in e-health, each local model refers to a model trained using the medical data of an individual patient. For each patient, the wearable device and its corresponding hospital possess a disjoint set of data features. Without sharing information between them, neither hospitals nor wearable devices can calculate partial derivatives for updating local models. Second, since raw data transmission may divulge patient privacy and cause high communication overhead, transmitting raw data between hospitals and wearable devices is infeasible. To calculate partial derivatives without disclosing private information of patients, we introduce an intermediate result exchange phase. During the model update process, hospitals and wearable devices communicate intermediate results which are determined by the raw data and trained models. Then each wearable device and its corresponding hospital can compute partial derivatives based on intermediate results to collaboratively update the local model. We will introduce the details of intermediate results in Section~\ref{sec:Problem_Formulation}.

\begin{figure}[!t]
	\centering
	\includegraphics[width=\columnwidth]{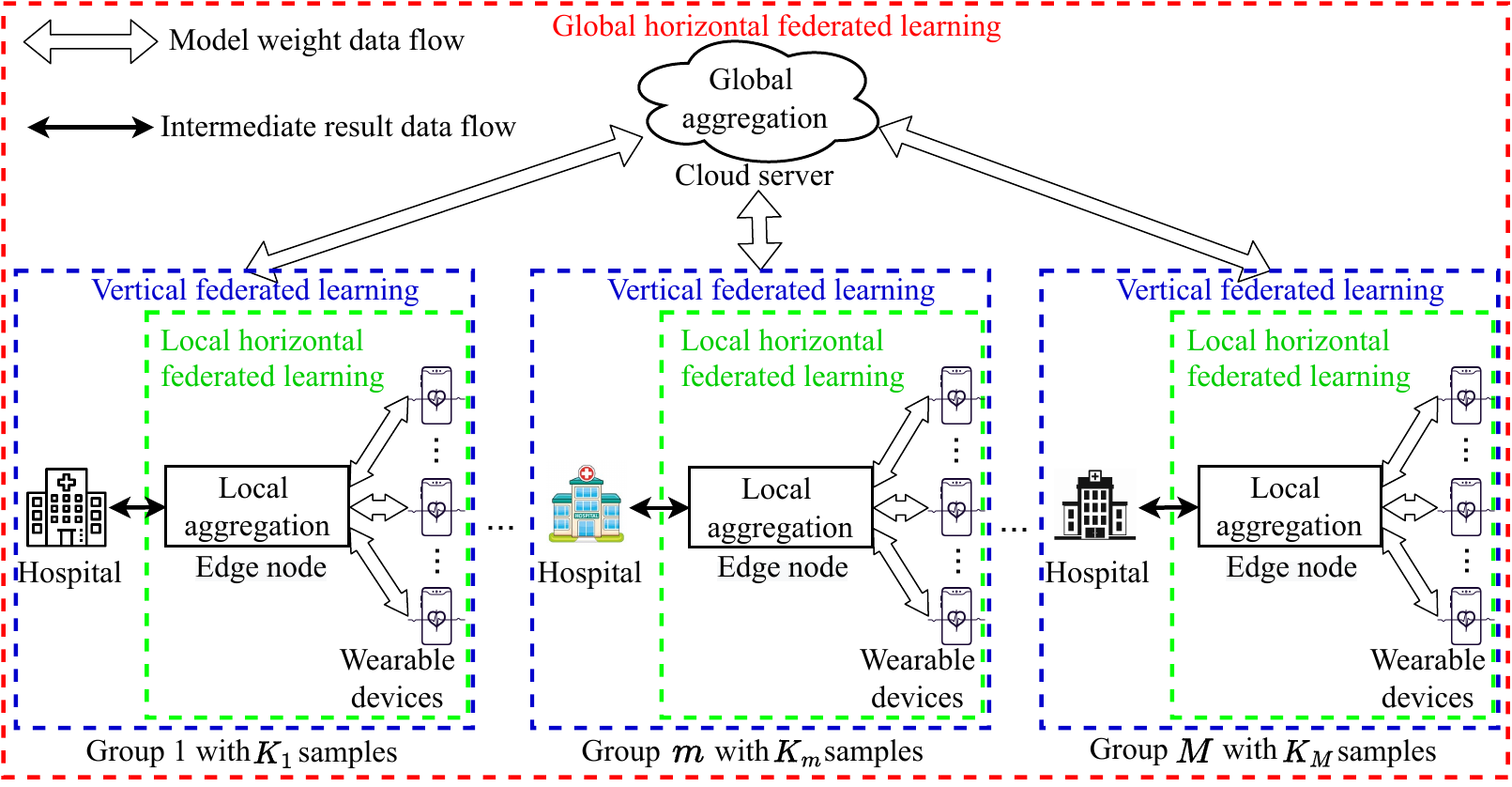} 
	\captionsetup{font={scriptsize}}
	\caption{Hybrid federated learning framework.}
	\label{Hybrid_federated_learning_framework}
\end{figure}

\textbf{Local aggregation phase.} Local aggregation is required to enhance training efficiency. Specifically, in each hospital-patient group, one hospital corresponds to multiple wearable devices. Allowing the hospital to learn a unique model for every wearable device may result in heavy computational burdens and increase training time. To improve training efficiency, the edge node, such as the base station, collects models trained on wearable devices within the same group to aggregate a uniform local device side model. The edge nodes may be placed in proximity to the devices that actively participate in the federated learning process. By introducing this local aggregation process, each hospital only trains one model that corresponds to the local device side model. Moreover, involving all wearable devices to participate in local aggregation may cause high communication overhead, while including a few wearable devices in this process can lead to overfitting due to limited samples. To address this, we randomly choose a subset of wearable devices to participate in the local aggregation~\cite{Zhang2022userselection}.

\textbf{Global aggregation phase.} Global aggregation is introduced to balance heterogeneous data and solve the issue of insufficient samples. Since the location, type, and reputation of hospitals affect the number of patients visiting hospitals, sub-datasets held by hospital-patient groups are diverse regarding data distribution. In addition, the number of samples for a hospital-patient group is limited due to privacy concerns and high data collection overhead. In these circumstances, local models trained on a single sub-dataset may overfit the local data~\cite{Shen2021FL}. To obtain a generalized model by leveraging the data collected by multiple hospital-patient groups, local models are sent to and aggregated on the cloud server.

\subsection{Problem Formulation}
\label{sec:Problem_Formulation}
For each hospital-patient group $m$, the local model consists of three sub-models: i) combined model; ii) hospital side model; iii) device side model. We define ${{\bm\theta_m}:=[({\bm\theta}_{0,m})^{\mathrm{T}},({\bm\theta}_{1,m})^{\mathrm{T}},({\bm\theta}_{2,m})^{\mathrm{T}}]^{\mathrm{T}}}$ to represent the local model parameter, where parameters ${\bm\theta}_{0,m}$, ${\bm\theta}_{1,m}$, and ${\bm\theta}_{2,m}$ are corresponding to combined model, hospital side model, and device side model, respectively. The hospital side model such as Convolutional Neural Network (CNN) or Long Short-Term Memory (LSTM) parameterized by ${\bm\theta}_{1,m}$ is used to train the data stored in hospitals. It maps $\mathbf X_1^{(m,n)}$ to ${\bm \zeta_{1,m,n}=h_1\left(\bm \theta_{1,m};\mathbf X_1^{(m,n)}\right)}$.  Similarly, device side model with parameter ${\bm\theta}_{2,m}$ maps $\mathbf X_2^{(m,n)}$ to ${\bm \zeta_{2,m,n}=h_2\left(\bm \theta_{2,m};\mathbf X_2^{(m,n)}\right)}$, which is utilized to train the data kept by wearable devices. Here, the outputs of the hospital side model and device side model, i.e., ${\bm \zeta}_{1,m,n}$ and ${\bm \zeta}_{2,m,n}$, are defined as intermediate results~\cite{Gu2022Privacy}. Then, the intermediate results are used to train a combined model with parameter ${\bm\theta}_{0,m}$, which can combine all data and output the final predictions. We assume that both the combined model and hospital side model are learned in the hospital.

As mentioned in Section~\ref{sec:Hybrid_federated_learning_framework}, the device side has an additional local aggregation phase compared to the hospital side. Based on this, we define
\begin{equation}
\begin{aligned}
{\bm\theta}_{2,m}=\frac{1}{|\mathcal{A}_m|}\sum_{n\in{\mathcal{A}_m}}{\bm\theta}_{2,m,n},
\label{LocalAggregation}
\end{aligned}
\end{equation}
where ${\bm\theta}_{2,m,n}$ is the model parameter learned on the \text{$n$-th} device in the $m$-th hospital-patient group, and the subset $\mathcal{A}_m$ consists of wearable devices that are selected to participate in local aggregation. We assume that a fixed proportion of wearable devices is chosen. It means that $|\mathcal{A}_m|=\alpha K_m,\forall m \in\{1, ..., M\}$, where $|\mathcal{A}_m|$ is the cardinality of subset $\mathcal{A}_m$.

The global model~${\tilde{\bm\theta}:=[(\tilde{\bm\theta}_0)^{\mathrm{T}}, (\tilde{\bm\theta}_1)^{\mathrm{T}},(\tilde{\bm\theta}_2)^{\mathrm{T}}]^{\mathrm{T}}}$ is the weighted average of local models, i.e.,
\begin{equation}
\begin{aligned}
&\tilde{\bm\theta}_i =\frac{1}{K}\sum_{m=1}^{M}K_m {\bm\theta}_{i,m}, \forall i=0,1,2.
\label{GlobalModel}	
\end{aligned}
\end{equation}
The loss on a sub-dataset $D_m$ with $K_m$ samples is
\begin{equation}
\begin{aligned}
&{F_{m}(\tilde{\bm\theta})}\!=\!\frac{1}{K_m}\!\sum_{n=1}^{K_m}f\left(\tilde{\bm\theta}_{0}, \tilde {\bm \zeta}_{1,m,n}, \tilde{\bm \zeta}_{2,m,n};y^{(m,n)}\right)\!+\!\sum_{i=1}^2r\!\left(\tilde{\bm\theta}_{i}\right),\\
&\text{with}\ \tilde{\bm \zeta}_{1,m,n}=h_1\left(\tilde{\bm\theta}_1;\mathbf X_1^{(m,n)}\right), \tilde{\bm \zeta}_{2,m,n}=h_2\left(\tilde{\bm\theta}_2;\mathbf X_2^{(m,n)}\right),
\label{LocalLossFunction}	
\end{aligned}
\end{equation}
where $f(\cdot)$ denotes loss function and $r(\cdot)$ represents the regularizer. The entire dataset is ${D=\cup_{m=1}^M D_m}$ which contains $K$ samples, and the global loss on dataset~$D$ is 
\begin{equation}
\begin{aligned}
&{F(\tilde{\bm\theta})}=\frac{1}{K}\sum_{m=1}^{M}K_m {F_{m}(\tilde{\bm\theta})}.
\label{GlobalLossFunction}	
\end{aligned}
\end{equation}
Our objective is to learn an optimal~$\tilde{\bm\theta}$ to minimize the global loss. The details of how to train the optimal global model are introduced in the next section.

\section{Hybrid Stochastic Gradient Descent Algorithm}
\label{sec:Hybrid_Stochastic_Gradient_Descent_Algorithm}
Based on the hybrid federated learning framework, we propose a HSGD algorithm for model training in e-health. The HSGD algorithm is shown as Algorithm~\ref{Algorithm1}.

In HSGD algorithm, we conduct global aggregation every $P$ iterations and local aggregation every $Q$ iterations. Each global aggregation includes $\Lambda$ local aggregations, i,e., $\frac{P}{Q}=\Lambda$, where $\Lambda$ is a positive integer. We define $t_0$ denotes the last iteration when the server aggregates the global model and $t_0^{\lambda}$ represents the last iteration to execute local aggregation. The relation between $t_0$ and $t_0^{\lambda}$ can be described as ${t_0^{\lambda}=t_0+\lambda Q, \lambda=\lfloor \frac{t-t_0}{Q}\rfloor}$. Then each global aggregation interval $[t_0,t_0+P-1]$ can be divided into multiple local aggregation intervals, i.e., $\cup_{\lambda=0}^{\Lambda-1}[t_0^{\lambda},t_0^{\lambda+1}-1]$ with $t_0^{0}=t_0$.

At the cloud server, an initial global model ${\tilde{\bm\theta}^{0}:=[(\tilde{\bm\theta}^0_0)^{\mathrm{T}}, (\tilde{\bm\theta}^0_1)^{\mathrm{T}},(\tilde{\bm\theta}^0_2)^{\mathrm{T}}]^{\mathrm{T}}}$ is generated. Model parameters $\tilde{\bm\theta}^0_0$ and $\tilde{\bm\theta}^0_1$ are transmitted to and updated in hospitals, and model parameter $\tilde{\bm\theta}^0_2$ is sent to and trained in wearable devices. In iteration 0, and every $P$-th iteration thereafter, the cloud server collects local models~${\bm\theta^{t}_m}, \forall m \in\{1, ..., M\}$. Then it aggregates a global model $\tilde{\bm\theta^{t}}$ by utilizing~\eqref{GlobalModel} (Line 4 of Algorithm~\ref{Algorithm1}). After that, the updated global model is transferred to hospitals and wearable devices again for iteratively training local models (Lines 5--9 of Algorithm~\ref{Algorithm1}).

\begin{algorithm}[t]
\caption{HSGD algorithm}
\label{Algorithm1}
\LinesNumbered
\KwIn {$P$, $Q$, $\eta$}
\KwOut {Global model $\tilde{\bm\theta}^{t}:=[(\tilde{\bm\theta}_0^{t})^{\mathrm{T}}, (\tilde{\bm\theta}_1^{t})^{\mathrm{T}},(\tilde{\bm\theta}_2^{t})^{\mathrm{T}}]^{\mathrm{T}}$}
Initialize ${\tilde{\bm\theta}^{0}:=[(\tilde{\bm\theta}^0_0)^{\mathrm{T}}, (\tilde{\bm\theta}^0_1)^{\mathrm{T}},(\tilde{\bm\theta}^0_2)^{\mathrm{T}}]^{\mathrm{T}}}$ \;
\For{$t=0,\ldots,T$}{   
    \If{$t\ (mod \ P)=0$} {
        The server computes $\tilde{\bm\theta}^t$ using~\eqref{GlobalModel}\;
        \For{$m=1,\ldots,M$ \textit{in parallel}} {
        ${\bm\theta}_{0,m}^t = \tilde{\bm\theta}_0^t$\;
        ${\bm\theta}_{1,m}^t = \tilde{\bm\theta}_1^t$\;
            \For{$n=1,\ldots,K_m$ \textit{in parallel}}{
                ${\bm\theta}_{2,m,n}^t = \tilde{\bm\theta}_2^t$\;
			}
        }
	}	 
	\If{$t\ (mod \ Q)=0$}{
		\For{$m=1,\ldots,M$ \textit{in parallel}}{
		    Edge nodes compute ${\bm\theta}_{2,m}^t$ utilizing~\eqref{LocalAggregation}\;
		    Each pair hospital and edge node agree on a subset $\mathcal{A}_m^{t_0^{\lambda}}$ and a corresponding mini-batch ${\xi}_m^{t_0^{\lambda}}$\;
		    \For{$n=1,\ldots,K_m$ \textit{in parallel}}{
		        ${\bm\theta}_{2,m,n}^t = {\bm\theta}_{2,m}^t$\;
		        \If{$n\in {\mathcal{A}_m^{t_0^{\lambda}}}$ \textit{in parallel}}{
		            Send ${\bm \zeta}_{2,m,n}^{t_0^{\lambda}}$ to edge nodes
		        }
		    }
		    Hospitals and edge nodes exchange intermediate results ${\bm\theta}^{t_0^{\lambda}}_{0,m}$, $\mathcal{Z}_{1,m}^{t_0^{\lambda}}$, and $\mathcal{Z}_{2,m}^{t_0^{\lambda}}$\;
		    \For{$n\in {\mathcal{A}_m^{t_0^{\lambda}}}$ \textit{in parallel}}{
		    Receive ${\bm\theta}^{t_0^{\lambda}}_{0,m}$ from edge nodes\;
		         Extract the information corresponding to its own samples from $\mathcal{Z}_{1,m}^{t_0^{\lambda}}$ \;
		    }
		}
	}
	\For{$m=1,\ldots,M$ \textit{in parallel}}{
	Update ${\bm\theta}_{0,m}^{t+1}$ using~\eqref{CombinedModelUpdate}\;
	Update ${\bm\theta}_{1,m}^{t+1}$ applying~\eqref{DeviceModelUpdate}\;
		\For{$n\in \mathcal{A}_m^{t_0^{\lambda}}$ \textit{in parallel}}{
		    Update ${\bm\theta}_{2,m,n}^{t+1}$ utilizing~\eqref{HospitalModelUpdate}\;
		}
	}
}
\end{algorithm}

At edge nodes, the local aggregation is conducted every~$Q$ iterations. In iteration 0, and every \text{$Q$-th} iteration thereafter, the edge node in hospital-patient group~$m$ calculates~${\bm\theta}_{2,m}$ based on~\eqref{LocalAggregation} (Line~12 of Algorithm~\ref{Algorithm1}). Then the edge node and hospital within hospital-patient group~$m$ agree on a wearable device subset $\mathcal{A}_m^{t_0^{\lambda}}$ and a corresponding mini-batch ${\xi}_m^{t_0^{\lambda}}, \lambda=\lfloor \frac{t-t_0}{Q}\rfloor$ (Line~13 of Algorithm~\ref{Algorithm1}), where $t_0$ denotes the last iteration when the server aggregates the global model, $t_0^{\lambda}$ is the last iteration that satisfies ${t_0^{\lambda}~(mod~Q)=0}$. At iteration $t_0^{\lambda}$, the aggregated model ${\bm\theta}_{2,m}$ is returned to all wearable devices within hospital-patient group $m$ for further model training (Line~15 of Algorithm~\ref{Algorithm1}). Note that the combined model~${\bm\theta}_{0,m}$ and hospital side model ${\bm\theta}_{1,m}$ do not need to be locally aggregated because each hospital-patient group only has one combined model and one hospital side model.

At wearable devices and their corresponding hospitals, intermediate results are communicated every $Q$ iterations for updating local models. From~\eqref{LocalLossFunction}, we know that ${\bm\theta}_{0,m}$, $\bm \zeta_{1,m,n}$, and $\bm \zeta_{2,m,n}$ are required to calculate the local loss. However, the wearable devices and hospitals cannot compute partial derivatives using their own data due to missing ${\bm\theta}_{0,m}$, $\bm \zeta_{1,m,n}$, or $\bm \zeta_{2,m,n}$. In this situation, the intermediate results and the combined model are shared between hospitals and wearable devices. The intermediate results $\bm \zeta_{1,m,n}$ and $\bm \zeta_{2,m,n}$ are calculated based on the raw data features but do not directly reveal the specific raw features themselves. Since hospitals hold ${\bm\theta}_{0,m}$ and $\bm \zeta_{1,m,n}$, they only request the intermediate result~$\bm \zeta_{2,m,n}$ from wearable devices. Conversely, wearable devices receive ${\bm\theta}_{0,m}$ and $\bm \zeta_{1,m,n}$ from their corresponding hospitals because they merely have $\bm \zeta_{2,m,n}$. To save communication cost, we use the same mini-batch to train local models during each local aggregation interval. It means that the intermediate results are calculated every $Q$ iterations and reused for the next $Q-1$ iterations. At iteration $t$, the intermediate result used on hospitals is ${\bm \zeta}^{t_0^{\lambda}}_{2,m,n}=h_2\left(\bm \theta^{t_0^{\lambda}}_{2,m};\mathbf X_2^{(m,n)}\right)$, while the intermediate results utilized on wearable devices are combined model ${\bm\theta}^{t_0^{\lambda}}_{0,m}$ and ${\bm \zeta}^{t_0^{\lambda}}_{1,m,n}=h_1\left(\bm \theta^{t_0^{\lambda}}_{1,m};\mathbf X_1^{(m,n)}\right)$.

\begin{figure}[!t]
	\centering
	\includegraphics[width=\columnwidth]{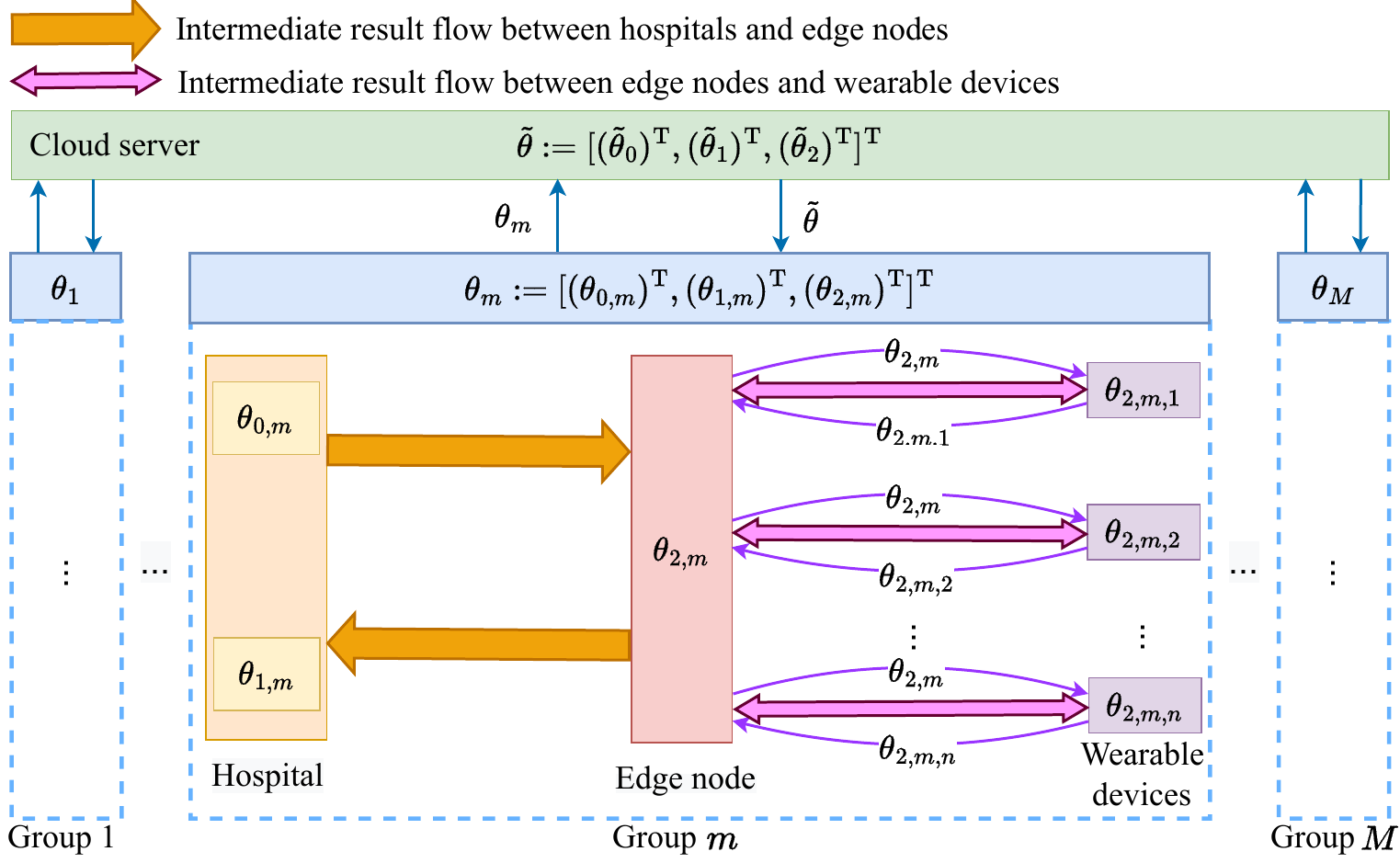} 
	\captionsetup{font={scriptsize}}
	\caption{HSGD algorithm implementation: the local model of hospital-patient group $m$ is ${{\bm\theta_m}:=[({\bm\theta}_{0,m})^{\mathrm{T}},({\bm\theta}_{1,m})^{\mathrm{T}},({\bm\theta}_{2,m})^{\mathrm{T}}]^{\mathrm{T}}}$, where ${\bm\theta}_{0,m}$ and ${\bm\theta}_{1,m}$ are trained on the hospital. Edge nodes not only aggregate models ${\bm\theta}_{2,m,n}, \forall n\in \mathcal{A}_m$ trained on wearable devices but also forward intermediate results for hospitals and wearable devices. The global model is generated on the cloud server by aggregating local models.}
	\label{HSGD_diagram}
    \vspace{-0.35cm}
\end{figure}

At iteration 0, and every $Q$-th iteration thereafter, hospitals and edge nodes generate the intermediate result for a single mini-batch ${\xi}_m^{t_0^{\lambda}}$. Each hospital $m$ calculates ${\mathcal{Z}_{1,m}^{t_0^{\lambda}}:=\{{\bm \zeta}_{1,m,n}^{t_0^{\lambda}}\}_{n\in {\xi}_m^{t_0^{\lambda}}}}$. Meanwhile, edge nodes receive ${\bm \zeta}^{t_0^{\lambda}}_{2,m,n}$ from wearable devices within the subset $\mathcal{A}^{t_0^{\lambda}}_m$ and stack them to form a intermediate result set $\mathcal{Z}_{2,m}^{t_0^{\lambda}}:=\{{\bm \zeta}_{2,m,n}^{t_0^{\lambda}}\}_{{n\in {\xi}_m^{t_0^{\lambda}}}}$. After that, edge nodes and their corresponding hospitals exchange ${\bm\theta}^{t_0^{\lambda}}_{0,m}$, $\mathcal{Z}_{1,m}^{t_0^{\lambda}}$, and $\mathcal{Z}_{2,m}^{t_0^{\lambda}}$ (Line~18 of Algorithm~\ref{Algorithm1}). For wearable device $n$ in subset ${\mathcal{A}_m^{t_0^{\lambda}}}$, it receives ${\bm\theta}^{t_0^{\lambda}}_{0,m}$ from edge nodes and extracts ${\bm \zeta}^{t_0^{\lambda}}_{1,m,n}$ corresponding to its own samples from $\mathcal{Z}_{1,m}^{t_0^{\lambda}}$ (Lines~19--21 of Algorithm~\ref{Algorithm1}).

At iteration $t$, the stochastic partial derivatives of local loss $F_m$ with respect to ${\bm\theta}_{i,m}^t$ is $\bm g_{i,m}({\bm\theta}_{0,m}^{t},\bm \zeta_{1,m,n}^{t}, {\bm \zeta}_{2,m,n}^{t_0^{\lambda}}; {\xi}_m^{t_0^{\lambda}})$ when ${i=0,1}$, which can be rewritten (with a slight abuse of notation) as a function of model parameters, i.e., $\bm g_{i,m}({\bm\theta}_{0,m}^{t},{\bm\theta}_{1,m}^{t}, {\bm\theta}_{2,m}^{t_0^{\lambda}}; {\xi}_m^{t_0^{\lambda}}), \forall i \in \{0,1\}$. Identically, the stochastic partial derivatives of $F_m$ with respect to ${\bm\theta}_{2,m,n}^t$ is $\bm g_{2,m,n}({\bm\theta}_{0,m}^{t_0^{\lambda}}, {\bm \zeta}_{1,m,n}^{t_0^{\lambda}}, \bm \zeta_{2,m,n}^{t}; {\xi}_{m,n}^{t_0^{\lambda}})$, which can be rewritten (with a slight abuse of notation) as $\bm g_{2,m,n}({\bm\theta}_{0,m}^{t_0^{\lambda}}, {\bm\theta}_{1,m}^{t_0^{\lambda}}, {\bm\theta}_{2,m,n}^{t}; {\xi}_{m,n}^{t_0^{\lambda}})$. Between two intermediate result exchange operations, hospitals and wearable devices execute $Q$ stochastic gradient steps to update models (Lines~\text{22--26} of Algorithm~\ref{Algorithm1}), i.e.,
\begin{align}
\label{CombinedModelUpdate}
{}&{\bm\theta}_{0,m}^{t+1} = {\bm\theta}_{0,m}^t - \eta \bm g_{0,m}({\bm\theta}_{0,m}^{t},{\bm\theta}_{1,m}^{t}, {\bm\theta}_{2,m}^{t_0^{\lambda}}; {\xi}_m^{t_0^{\lambda}}),\\
\label{DeviceModelUpdate}
{}&{\bm\theta}_{1,m}^{t+1} = {\bm\theta}_{1,m}^t - \eta \bm g_{1,m}({\bm\theta}_{0,m}^{t},{\bm\theta}_{1,m}^{t}, {\bm\theta}_{2,m}^{t_0^{\lambda}}; {\xi}_m^{t_0^{\lambda}}),\\
 \label{HospitalModelUpdate}	
{}&{\bm\theta}_{2,m,n}^{t+1} \!=\! {\bm\theta}_{2,m,n}^t \!- \!\eta \bm g_{2,m,n}({\bm\theta}_{0,m}^{t_0^{\lambda}}, {\bm\theta}_{1,m}^{t_0^{\lambda}}, {\bm\theta}_{2,m,n}^{t}; {\xi}_{m,n}^{t_0^{\lambda}}),
\end{align}
where $\eta$ represents the learning rate, which is set to be equal for hospitals and wearable devices. The training process is repeatedly conducted until convergence.  The diagram for HSGD algorithm implementation is shown in Fig.~\ref{HSGD_diagram}.

\textbf{Privacy}: In our proposed HSGD algorithm, both the local and global aggregation processes involve transmitting the model parameters, which are also shared in standard HFL. This ensures the same level of privacy preservation as HFL. Additionally, HSGD algorithm exchanges intermediate results, which are typically communicated in standard VFL, thereby providing the same privacy guarantees as VFL. Although HSGD algorithm does not transfer raw data during the training process, it may be vulnerable to certain attacks as the standard federated learning. Gradient attacks~\cite{geiping2020inverting} may infer the private information from the model parameters, and model inversion attacks~\cite{mahendran2015understanding} can potentially recover raw data based on the intermediate results. To address these concerns, various techniques such as homomorphic encryption~\cite{zhang2022homomorphic,zhang2020batchcrypt,hardy2017private} and differential privacy~\cite{wei2020federated,jiang2021privacy} have been applied in HFL or VFL settings to resist these attacks. Our proposed HSGD algorithm can also integrate with these secure mechanisms to protect patients' privacy and the integrity of their data.

\section{Convergence analysis}
\label{sec:Convergence_analysis}
In this section, we theoretically analyze the convergence of the proposed HSGD algorithm. We make the following assumptions~~\cite{Wang2019AFL,Zhang2022FLAssumptions,Suhas2020FLAssumptions}.

\begin{assumption}\label{Assumption1} The local gradient is Lipschitz continuous with constant $\rho$; further, the local stochastic partial derivatives are Lipschitz continuous with constant $\rho_i$, i.e., for  model parameters $\bm\theta$ and $\bm\upsilon$,
\begin{equation}
\begin{aligned}
\left\Vert\nabla F_m(\bm\theta)-\nabla F_m(\bm\upsilon)\right\Vert \leq \rho\left\Vert \bm\theta-\bm\upsilon\right\Vert,
\label{EqAssumption1_1}	
\end{aligned}
\end{equation}
\begin{equation}
\begin{aligned}
\left\Vert  \bm g_{i,m}(\bm\theta)- \bm g_{i,m}(\bm\upsilon)\right\Vert \leq \rho_i\left\Vert \bm\theta-\bm\upsilon\right\Vert, \forall i = 0, 1,
\label{EqAssumption1_2}	
\end{aligned}
\end{equation}
\begin{equation}
\begin{aligned}
\left\Vert  \bm g_{2,m,n}(\bm\theta)- \bm g_{2,m,n}(\bm\upsilon)\right\Vert \leq \rho_2\left\Vert \bm\theta-\bm\upsilon\right\Vert.
\label{EqAssumption1_3}	
\end{aligned}
\end{equation}
where ${\rho \geq \max_{0\leq i \leq 2}\ \rho_i}$.
\end{assumption}

\begin{assumption}\label{Assumption2} The local stochastic partial gradient is unbiased with bounded variance, i.e., 
\begin{equation}
\begin{aligned}
\mathbb{E}[\bm g_{i,m}(\bm\theta)\mid\bm\theta]= \nabla_{(i)} F(\bm\theta), \forall i=0,1,
\label{EqAssumption2_1}	
\end{aligned}
\end{equation}
\begin{equation}
\begin{aligned}
\mathbb{E}[\bm g_{2,m,n}(\bm\theta)\mid\bm\theta]= \nabla_{(2)} F(\bm\theta),
\label{EqAssumption2_2}	
\end{aligned}
\end{equation}
\begin{equation}
\begin{aligned}
\mathbb{E}\left[\left\Vert \bm g_{i,m}(\bm\theta)-\nabla_{(i)} F(\bm\theta)\right\Vert ^2 \mid \bm\theta\right]\leq \delta^2, \forall i=0,1,
\label{EqAssumption2_3}	
\end{aligned}
\end{equation}
\begin{equation}
\begin{aligned}
\mathbb{E}\left[\left\Vert \bm g_{2,m,n}(\bm\theta)-\nabla_{(2)} F(\bm\theta)\right\Vert ^2 \mid \bm\theta\right]\leq \delta^2.
\label{EqAssumption2_4}	
\end{aligned}
\end{equation}
\end{assumption}

The global model parameter $\tilde{\bm\theta}$ is observable only when global aggregation is conducted, i.e., $t~(mod~P)=0$, but to facilitate the analysis, we define
\begin{equation}
\begin{aligned}
\tilde{\bm\theta}^{t+1} = \tilde{\bm\theta}^t-\eta \bm G^t
\label{GlobalUpdate}	
\end{aligned}
\end{equation}
for all $t$, where ${\bm G^t}$ denotes the global gradient, i.e.,
\begin{equation}
\begin{aligned}
&{\bm G^t}\!=\![({\bm G^t}_{(0)})^{\mathrm{T}}, ({\bm G^t}_{(1)})^{\mathrm{T}},({\bm G^t}_{(2)})^{\mathrm{T}}]^{\mathrm{T}}\\
&\ \ \ \ \!\!=\!\!\left [\!\!\!\begin{array}{l}
\frac{1}{K}\sum_{m=1}^{M}\!K_m \bm g_{0,m}({\bm\theta}_{0,m}^{t},{\bm\theta}_{1,m}^{t}, {\bm\theta}_{2,m}^{t}; {\xi}_m^{t})\\
\frac{1}{K}\sum_{m=1}^{M}\!K_m \bm g_{1,m}({\bm\theta}_{0,m}^{t},{\bm\theta}_{1,m}^{t}, {\bm\theta}_{2,m}^{t}; {\xi}_m^{t})\\
\frac{1}{K}\sum_{m=1}^{M}\!K_m {\bm G^t}_{(2,m)}
\end{array}
\!\!\!\right ]\!\!,\\
&{\bm G^t}_{(2,m)}\! = \!\frac{1}{\vert \mathcal{A}_m^t \vert}\!\!\sum_{n\in{\mathcal{A}_m^t}}\!\!\!\bm g_{2,m,n}({\bm\theta}_{0,m}^{t}, {\bm\theta}_{1,m}^{t}, {\bm\theta}_{2,m,n}^{t}; {\xi}_{m,n}^{t}).
\label{GlobalGradient}	
\end{aligned}
\end{equation}

We present the main theoretical result of this paper. The result proves that our proposed algorithm can converge under certain conditions. Note that our convergence analysis only focuses on the iid data distribution, whereas non-iid data distribution is considered in our experiments. The proof of the theorem is deferred to Appendix~A.

\begin{theorem}\label{Theorem1}
Under Assumptions~\ref{Assumption1}--\ref{Assumption2}, when the learning rate $\eta$ satisfies $\eta\leq \frac{1}{8P\rho}$, the expected averaged squared gradient of $F$ over $T=RP$ is upper bounded by
\begin{equation}
\label{EqTheorem1}
\begin{aligned}
&\mathbb{E}\left[\frac{1}{R}\sum_{r=0}^{R-1} \left\Vert\nabla F(\tilde{\bm\theta}^{rP})\right\Vert^2\right]\\
&\leq \frac{4\left(F(\tilde{\bm\theta}^{0})-\mathbb{E}[F(\tilde{\bm\theta}^{T})]\right)}{\eta T} + 12P\rho\eta\delta^2 + 96Q^2\rho^2\eta^2 \delta^2.
\end{aligned}
\end{equation}
\end{theorem}

Theorem~\ref{Theorem1} shows that $\mathbb{E}\left[\frac{1}{R}\sum_{r=0}^{R-1} \left\Vert\nabla F(\tilde{\bm\theta}^{rP})\right\Vert^2\right]$ converges to zero when choosing $\eta=\Theta\left(\frac{1}{\sqrt{T}}\right)$. Based on this, we can conclude that the proposed HSGD algorithm converges as the total number of iteration $T$ gets large.

\section{Adaptive Strategies for Balancing Convergence Bound and Communication Cost}
\label{sec:Strategies_for_Minimizing_Global_Loss_Function}

We theoretically analyze the influence of several important training parameters for the HSGD algorithm on convergence upper bound when considering communication cost. Insightful strategies are given based on the analysis and validated by experimental results.

\subsection{Adaptive Strategy 1 for Determining the Relation of Global and Local Aggregation Intervals}
When considering communication cost, we first analyze how the difference between global aggregation interval $P$ and local aggregation interval $Q$ affects convergence upper bound. Then an adaptive strategy to determine the relation between $P$ and $Q$ is provided for minimizing communication cost while reaching expected upper bound.
\begin{proposition}\label{Proposition1}\renewcommand{\qedsymbol}{}
For the HSGD algorithm, communication cost increases with the ratio of $\frac{P}{Q}$ when guaranteeing the expected convergence bound, and the minimum communication cost can be achieved by setting $P=Q$ in certain cases.
\end{proposition}

\begin{proof}
We use $\Gamma(P,Q)$ to represent the upper bound (i.e., right-hand side) in~\eqref{EqTheorem1}. The expected convergence upper bound is defined as $\Xi$. To achieve this expected bound, we have $\Gamma(P,Q)\leq \Xi$. Substituting $Q=\frac{P}{\Lambda}$ into it, we get 
\begin{equation}
\label{Proposition1Proof1}
\begin{aligned}
\frac{4\left(F(\tilde{\bm\theta}^{0})-\mathbb{E}[F(\tilde{\bm\theta}^{T})]\right)}{\eta T} + 12P\rho\eta\delta^2 + 96\left(\frac{P}{\Lambda}\right)^2\rho^2\eta^2 \delta^2 \leq \Xi.
\end{aligned}
\end{equation}
From~\eqref{Proposition1Proof1}, we obtain the relation between $P$ and $\Lambda$, i.e., ${\Lambda \geq \frac{4\sqrt{6}P\rho\eta\delta}{\sqrt{\Xi-\frac{4\left(F(\tilde{\bm\theta}^{0})-\mathbb{E}[F(\tilde{\bm\theta}^{T})]\right)}{\eta T}-12P\rho\eta\delta^2}}}$.

The sizes of hospital side model parameter, device side model parameter, and combined model parameter are represented by $|{\bm\theta}_{1}|$, $|{\bm\theta}_{2}|$, and $|{\bm\theta}_{0}|$. Furthermore, the number of selected devices in each group is denoted by $|\mathcal{A}|$. We also define $|\mathcal{Z}_1^{t_0^{\lambda}}|$ to represent the size of intermediate results transmitted from hospitals to devices and $|\mathcal{Z}_2^{t_0^{\lambda}}|$ to denote the size of intermediate results transmitted from devices to hospitals. The total communication cost during $T$ iterations is related to the global aggregation interval $P$ and local aggregation interval $Q$ and can be represented by ${C(P,Q)=\left(\frac{|{\bm\theta}_{1}|}{P}+\frac{|\mathcal{A}|{|\bm\theta}_{2}|+|{\bm\theta}_{0}|+|\mathcal{Z}_1^{t_0^{\lambda}}|+|\mathcal{Z}_2^{t_0^{\lambda}}|}{Q}\right)MT}$, where $M$ is the number of hospital-patient groups. According to the relation between $P$ and $Q$, we have 
\begin{equation}
\label{Proposition1Proof2}
\begin{aligned}
C(P,Q)\!\!=\!\!\left(\!\frac{|{\bm\theta}_{1}|}{P}+\frac{\Lambda\left(|\mathcal{A}|{|\bm\theta}_{2}|+|{\bm\theta}_{0}|+|\mathcal{Z}_1^{t_0^{\lambda}}|+|\mathcal{Z}_2^{t_0^{\lambda}}|\right)}{P}\!\right)\!\!MT
\end{aligned}
\end{equation}
From~\eqref{Proposition1Proof2}, we observe that the communication cost increases with $\Lambda$ when $P$ is fixed.

Based on the findings from~\eqref{Proposition1Proof1} and~\eqref{Proposition1Proof2}, it is evident that the value of $\Lambda$ is lower bounded, and the communication cost is positively correlated with $\Lambda$. Consequently, a smaller value of $\Lambda$ can contribute to reducing communication costs, indicating a narrower gap between $P$ and $Q$. The minimum communication cost is achieved at ${\Lambda = \frac{4\sqrt{6}P\rho\eta\delta}{\sqrt{\Xi-\frac{4\left(F(\tilde{\bm\theta}^{0})-\mathbb{E}[F(\tilde{\bm\theta}^{T})]\right)}{\eta T}-12P\rho\eta\delta^2}}}$.
When carefully setting $\eta$ to satisfy ${1\geq\frac{4\sqrt{6}P\rho\eta\delta}{\sqrt{\Xi-\frac{4\left(F(\tilde{\bm\theta}^{0})-\mathbb{E}[F(\tilde{\bm\theta}^{T})]\right)}{\eta T}-12P\rho\eta\delta^2}}}$, the optimal $\Lambda$ equals $1$, i.e., $P=Q$.
\end{proof}

\textit{Adaptive strategy 1}: According to Proposition~\ref{Proposition1}, we propose an adaptive strategy 1, which gives the effect of the difference between $P$ and $Q$ on the training performance. Specifically, to minimize the communication cost for achieving a target training requirement, the global aggregation interval $P$ and local aggregation interval $Q$ should be set to the same with carefully selected $\eta$.

\subsection{Adaptive Strategy 2 for Optimizing Global and Local Aggregation Intervals}
We next analyze how local aggregation interval $Q$ affects the convergence upper bound~\eqref{EqTheorem1} and communication cost when Proposition~\ref{Proposition1} is applied. Based on the analysis, we propose an adaptive strategy to optimize $P$ and $Q$ for achieving the trade-off between convergence bound and communication cost.

\begin{proposition}\label{Proposition2}
For the HSGD algorithm with $P\!\!=\!\!Q$, the composite indicator of convergence bound~\eqref{EqTheorem1} and communication cost decreases first and then increases with $Q$, and the extreme point is ${Q=\sqrt{\frac{F(\tilde{\bm\theta}^{0})-\mathbb{E}[F(\tilde{\bm\theta}^{T})]}{24\rho^2\eta^2\delta^2T}}}$. The \text{trade-off} of the upper bound and communication cost is achieved by choosing ${P=Q=\sqrt{\frac{F(\tilde{\bm\theta}^{0})-\mathbb{E}[F(\tilde{\bm\theta}^{T})]}{24\rho^2\eta^2\delta^2T}}}$.
\end{proposition}

\begin{proof}
Based on Proposition~\ref{Proposition1}, we can substitute $P=Q$ into $\Gamma(P,Q)C(P,Q)$ and obtain $\Gamma(Q)C(Q)$. Then we compute the derivative of $\Gamma(Q)C(Q)$ with respect to $Q$, i.e.,
\begin{equation}
\label{Proposition2Proof1}
\begin{aligned}
\frac{\partial \Gamma(Q)C(Q)}{\partial Q}=&-\frac{4\left(F(\tilde{\bm\theta}^{0})-\mathbb{E}[F(\tilde{\bm\theta}^{T})]\right)}{\eta}ZQ^{-2}\\
&+96\rho^2\eta^2\delta^2ZT,
\end{aligned}
\end{equation}
where $Z =( |{\bm\theta}_{1}|+|\mathcal{A}|{|\bm\theta}_{2}|+|{\bm\theta}_{0}|+|\mathcal{Z}_1^{t_0^{\lambda}}|+|\mathcal{Z}_2^{t_0^{\lambda}}|)M$. 
Since ${F(\tilde{\bm\theta}^{0})-F(\tilde{\bm\theta}^{T})>0, Z >0, \rho > 0, \eta > 0, \delta>0}$, $Q > 0$, and $T > 0$, when $Q$ satisfies the following condition:
\begin{equation}
\label{Proposition2Proof1_1}
\begin{aligned}
0 <Q < \sqrt{\frac{F(\tilde{\bm\theta}^{0})-\mathbb{E}[F(\tilde{\bm\theta}^{T})]}{24\rho^2\eta^2\delta^2T}}
\end{aligned}
\end{equation}
we have ${\frac{\partial \Gamma(Q)C(Q)}{\partial Q} < 0}$. While when $Q$ satisfies that
\begin{equation}
\label{Proposition2Proof1_2}
\begin{aligned}
Q > \sqrt{\frac{F(\tilde{\bm\theta}^{0})-\mathbb{E}[F(\tilde{\bm\theta}^{T})]}{24\rho^2\eta^2\delta^2T}}
\end{aligned}
\end{equation}
we obtain ${\frac{\partial \Gamma(Q)C(Q)}{\partial Q} > 0}$. The result shows that the minimum value of $\Gamma(Q)C(Q)$ is achieved when setting ${Q = \sqrt{\frac{F(\tilde{\bm\theta}^{0})-\mathbb{E}[F(\tilde{\bm\theta}^{T})]}{24\rho^2\eta^2\delta^2T}}}$. 
\end{proof}

\textit{Adaptive strategy 2}: We propose an adaptive strategy 2 based on Proposition~\ref{Proposition2}. To achieve the trade-off between the convergence upper bound and communication cost, the global aggregation interval $P$ and local aggregation interval $Q$ should be set to be $\sqrt{\frac{F(\tilde{\bm\theta}^{0})-\mathbb{E}[F(\tilde{\bm\theta}^{T})]}{24\rho^2\eta^2\delta^2T}}$. Here $\mathbb{E}[F(\tilde{\bm\theta}^{T})]$ is the expected global loss at $T$-th iteration. We can evaluate unknown parameters including $F(\tilde{\bm\theta}^{0}), \rho$, and $\delta$ by performing a small number of pre-training but cannot know the optimal global loss in advance. Approximating $\mathbb{E}[F(\tilde{\bm\theta}^{T})]$ by 0, we propose an approximate optimal solution, i.e., $P^*=Q^*=\sqrt{\frac{F(\tilde{\bm\theta}^{0})}{24\rho^2\eta^2\delta^2T}}$.

\subsection{Adaptive Strategy 3 for Adjusting Learning Rate when Aggregation Intervals Change}
Theorem~\ref{Theorem1} shows the convergence upper bound~\eqref{EqTheorem1} is not only affected by $P$ and $Q$, but also influenced by learning rate~$\eta$. Learning rate is included in $-\frac{4\left(F(\tilde{\bm\theta}^{0})-\mathbb{E}[F(\tilde{\bm\theta}^{T})]\right)}{\eta T}$, i.e., the first item on the right-hand side of~\eqref{EqTheorem1}, where $\mathbb{E}[F(\tilde{\bm\theta}^{T})]$ denotes expected global loss at $T$-th iteration. Since $\mathbb{E}[F(\tilde{\bm\theta}^{T})]$ is determined by the training iteration $T$, it cannot be directly regarded as a constant when analyzing the effect of $\eta$. Therefore, we analyze the influence of $\eta$ on training performance utilizing the change of global loss within one global aggregation, i.e., $\mathbb{E}\left[F(\tilde{\bm\theta}^{t_0^\Lambda-1})\mid \tilde{\bm\theta}^{t_0}\right]-F(\tilde{\bm\theta}^{t_0})$. Based on the analysis, we give the following proposition, which shows how to adjust the learning rate $\eta$ for minimizing the upper bound of $\mathbb{E}\left[F(\tilde{\bm\theta}^{t_0^\Lambda-1})\mid \tilde{\bm\theta}^{t_0}\right]-F(\tilde{\bm\theta}^{t_0})$ and communication cost when $P$ or $Q$ changes.

\begin{proposition}\label{Proposition3}For HSGD algorithm with any fixed $Q$ or $\frac{P}{Q}$, the optimal learning rate $\eta$ that minimizes the upper bound of $\mathbb{E}\left[F(\tilde{\bm\theta}^{t_0^\Lambda-1})\mid \tilde{\bm\theta}^{t_0}\right]-F(\tilde{\bm\theta}^{t_0})$ and communication cost decreases with $P$ or $Q$, respectively.
\end{proposition}

\begin{proof}
For notation brevity, we define $\mathbb{E}^{t_0}:=\mathbb{E}\left[\ \bm\cdot\mid \tilde{\bm\theta}^{t_0}\right]$. According to~\eqref{GlobalLossFunction} 
and Assumption~\ref{Assumption1}, we know that global gradient follows \text{$\rho$-Lipschitz}. Based on this, we can obtain
\begin{equation}
\begin{aligned}
&\mathbb{E}^{t_0}\left[F(\tilde{\bm\theta}^{t_0^\Lambda-1})\right]- F(\tilde{\bm\theta}^{t_0})\\
&\leq -\mathbb{E}^{t_0}\left\langle \nabla F(\tilde{\bm\theta}^{t_0}), \eta \sum_{t=t_0}^{t_0^\Lambda-1} \bm G^{t}\right\rangle + \frac{\rho}{2}\mathbb{E}^{t_0}\left\Vert\eta \sum_{t=t_0}^{t_0^\Lambda-1} \bm G^{t}\right\Vert^2.
\label{Proposition3Proof1}
\end{aligned}
\end{equation}

Substituting Lemmas~1--6 (please refer to Appendix~A) into~\eqref{Proposition3Proof1}, we have,
\begin{equation}
\begin{aligned}
&\mathbb{E}^{t_0}\left[F(\tilde{\bm\theta}^{t_0^\Lambda-1})\right]- F(\tilde{\bm\theta}^{t_0})\\
&\leq 24\eta^3Q^2P\rho^2\delta^2+ 3\eta^2P^2\rho\delta^2 -\frac{\eta P}{4}\left\Vert\nabla F(\tilde{\bm\theta}^{t_0})\right\Vert^2.
\end{aligned}
\label{Proposition3Proof2}
\end{equation}
To simplify the expression, we define
\begin{equation}
\label{Proposition3Proof3}
\begin{aligned}
a:=24Q^2P\rho^2\delta^2; b:=3P^2\rho\delta^2; c:=\frac{P}{4}\left\Vert\nabla F(\tilde{\bm\theta}^{t_0})\right\Vert^2.
\end{aligned}
\end{equation}
Then we can rewrite~\eqref{Proposition3Proof2} as 
\begin{equation}
\begin{aligned}
\mathbb{E}^{t_0}\left[F(\tilde{\bm\theta}^{t_0^\Lambda-1})\right]- F(\tilde{\bm\theta}^{t_0})\leq a\eta^3+ b\eta^2 -c\eta.
\end{aligned}
\label{Proposition3Proof4}
\end{equation}
Here $a,b,c \geq 0$. We define $\digamma(P,Q)$ to denote the upper bound (i.e., right-hand side) in~\eqref{Proposition3Proof4} and $\frac{C(P,Q)P}{T}$ to represent the total communication cost in each global aggregation interval. To adjust learning rate $\eta$ to minimize upper bound $\digamma(P,Q)$ and communication cost in each global aggregation interval, we calculate the derivative of $\frac{\digamma(P,Q)C(P,Q)P}{T}$ with respect to $\eta$. Letting $\frac{\partial \frac{\digamma(P,Q)C(P,Q)P}{T}}{\partial \eta}=0$, we have two solutions $\eta_1$ and $\eta_2$, where $\eta_1=\frac{-2b-\sqrt{4b^2+12ac}}{6a}\leq 0$ and $\eta_2=\frac{-2b+\sqrt{4b^2+12ac}}{6a} \geq 0$. Derivative $\frac{\partial \frac{\digamma(P,Q)C(P,Q)P}{T}}{\partial \eta}$ is negative when $\eta_1< \eta <\eta_2$ and is positive when ${\eta >\eta_2}$ and ${\eta <\eta_1}$. Considering the range of $\eta$ given in Theorem~\ref{Theorem1}, we can obtain that $\mathbb{E}^{t_0}\left[F(\tilde{\bm\theta}^{t_0^\Lambda-1})\right]- F(\tilde{\bm\theta}^{t_0})$ and communication cost are minimized when ${\eta = \min \{\eta_2,\frac{1}{8P\rho}\}}$.

To evaluate effects of $P$ and $Q$ on the optimal $\eta$, we should calculate $\frac{\partial  \eta }{\partial P}$ and $\frac{\partial  \eta }{\partial Q}$. We take $\eta=\eta_2$ as an example to show the influence. Taking the derivative of the optimal $\eta$ with respect to $a$, we have 
$\frac{\partial \eta}{\partial a}=\frac{-3\left(\sqrt{4b^2+12ac}-2b\right)^2}{36a^2 \sqrt{4b^2+12ac}}$. Since $a,b,c \geq 0$, we have $\frac{\partial \eta}{\partial a} \leq 0$. Then deriving $\frac{\partial a}{\partial P}$, we obtain $\frac{\partial a}{\partial P} > 0$. The optimal learning rate $\eta$ decreases with $a$ because $\frac{\partial \eta}{\partial a} \leq 0$, and $a$ is a increase function of $P$ since $\frac{\partial a}{\partial P} > 0$, so that the optimal $\eta$ decreases with $P$. To evaluate the influence of~$Q$, we define that $P=\Lambda Q$, where $\Lambda$ is a positive integer. Using $P=\Lambda Q$ and computing derivative of $\frac{\partial a}{\partial Q}$, we have $\frac{\partial a}{\partial Q}> 0$. Due to $\frac{\partial \eta}{\partial a} <0$ and $\frac{\partial a}{\partial Q} >0$, we have $\frac{\partial \eta}{\partial Q} <0$. This means that when $\frac{P}{Q}$ is fixed, the increase of $Q$ would result in the optimal $\eta$ decrease.
\end{proof}

\textit{Adaptive strategy 3}: To minimize the communication cost for achieving target training requirements, we propose an adaptive strategy 3 based on the above analysis. Adaptive strategy 3 tunes learning rate $\eta$ when aggregation intervals $P$ and $Q$ vary. This strategy includes two parts: i) when $Q$ is fixed, learning rate~$\eta$ should decrease with $P$; ii) when $\frac{P}{Q}$ is fixed and $Q$ increases, learning rate $\eta$ should decrease.

\section{Experimental Evaluation}
\label{sec:Experimental_Evaluation}
To evaluate the proposed HSGD algorithm and adaptive strategies, we conduct experiments on different datasets, tasks, and models, and compare results with commonly used benchmarks and state-of-the-art methods. In this section, we first introduce our experimental setup including baselines, datasets, and parameter settings in Section~\ref{sec:Experimental_Settings}. Then in Section~\ref{sec:Efficiency_of_HSGD_algorithm}, we validate the effectiveness of the proposed HSGD algorithm in guaranteeing accuracy, and saving training time and communication cost. Finally, adaptive strategies for adjusting global and local aggregation intervals, and learning rates to achieve target accuracy while reducing communication cost are verified in Section~\ref{sec:Design_principle_of_HSGD_algorithm_validation}.

\subsection{Experimental Setup}
\label{sec:Experimental_Settings}

\textit{1) Baselines:} To the best of our knowledge, only two state-of-the-art federated learning techniques can deal with both horizontal and vertical data partitioning. For comparison, we utilize these two methods and their variations as baselines. To ensure a fair comparison, we employ a random participant selection scheme for all baselines. Additionally, unless stated otherwise, all other settings remain the same across the experiments to maintain consistency.

\begin{itemize}
	\item[i)] JFL~\cite{Yu2022FL}: This baseline combines HFL and VFL. When implementing JFL in e-health systems, each randomly selected wearable device (i.e., the farm edge node in~\cite{Yu2022FL}) and the corresponding hospital (i.e., the server in~\cite{Yu2022FL}) collaboratively learn a unique local model by conducting VFL. Next, local models are transmitted to a server for global aggregation. Then, aggregated results are sent back to wearable devices and hospitals for the next round of updates. Compared with the proposed HSGD algorithm, JFL lacks the local aggregation phase. 	
	\item[ii)]  TDCD~\cite{Das2021CSFLjournal}: It is a combination of HFL and VFL with a two-tier horizontal-vertical structure. Using TDCD, each randomly selected device (i.e., the client in~\cite{Das2021CSFLjournal}) in each group (i.e., the silo in~\cite{Das2021CSFLjournal}) performs local gradient steps before sharing updates with their edge nodes (i.e., the hubs in~\cite{Das2021CSFLjournal}). Each hub adjusts its coordinates by averaging its clients’ updates, and then exchanges intermediate updates with others. Note TDCD has no global aggregation phase compared to the proposed HSGD algorithm. In addition, e-health has a three-tier horizontal-vertical-horizontal data structure, while TDCD can deal with a two-tier horizontal-vertical data structure. To apply TDCD to e-health systems, we combine data in different hospital-patient groups to form one group. This step requires transmitting a part of the raw data. 
	\item[iii)]  C-TDCD: Castiglia et al.~\cite{castiglia2022compressed} proposed a Compressed-VFL method that enhances the communication efficiency of training with vertically partitioned data. We adopt this method to TDCD, where the top-$k$ sparsification is utilized to compress the intermediate results and the combined models. For simplifying expression, we refer to this method as C-TDCD.
	\item[iv)]  C-HSGD: For comparison, we also incorporate top-$k$ sparsification into the vertical training process between wearable devices and the corresponding hospitals in our proposed method, referred to as C-HSGD.
\end{itemize}

\textit{2) Datasets:} We conduct experiments on three datasets: OrganAMNIST, MIMIC-III, and Epileptic Seizure Recognition (ESR), where OrganAMNIST is an image dataset, MIMIC-III and ESR are time series datasets. Since \text{e-health} data can be either in image or time series format, using these three datasets can adequately show the generality of the proposed HSGD algorithm and adaptive strategies. OrganAMNIST consists of a set of ${28\times28}$ grayscale images from 11 classes. It contains 34581 training samples and 17778 test samples~\cite{Yang2021MedMNIST}. MIMIC-III is a public dataset comprising electronic health information of patients admitted to intensive care units in a medical center~\cite{Johnson2016MIMIC_III}. \text{MIMIC-III} can be preprocessed as in~\cite{MIMIC-III_preprocessed} to generate 14681 training samples and 3236 test samples including two classes. Each sample in MIMIC-III is represented as a time series with 48 time steps and 76 features. ESR is a preprocessed and restructured version of a commonly used dataset featuring epileptic seizure detection~\cite{miscepilepticseizurerecognition388}. It consists of 11500 samples with 178 features, including 5 levels of epilepsy. We use 80\% of the samples for training and the remaining samples for testing purposes.

\begin{figure*}[!t]
	\centering
	\centering
	\subfloat[OrganAMNIST]{\includegraphics[width=0.66\columnwidth]{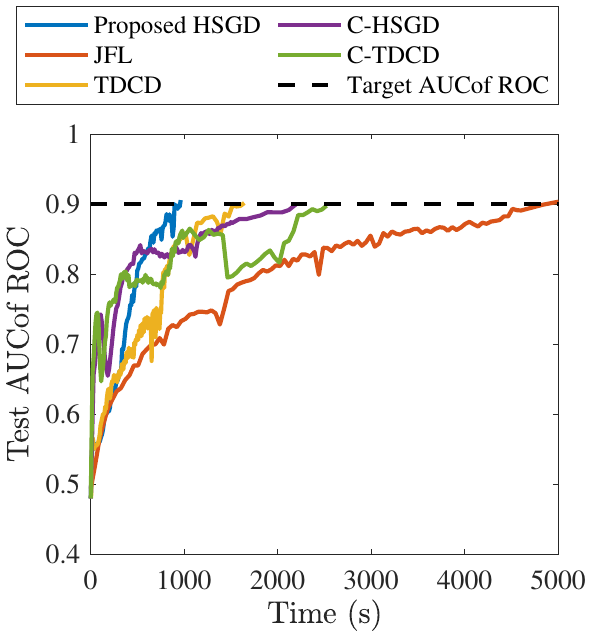}
	\label{Fig4a}}
	\subfloat[ESR]{\includegraphics[width=0.67\columnwidth]{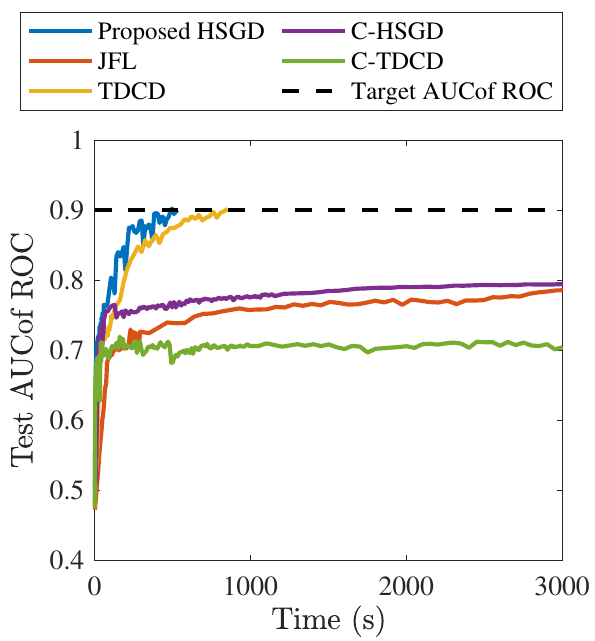}
	\label{Fig4b}}
	\subfloat[MIMIC-III]{\includegraphics[width=0.66\columnwidth]{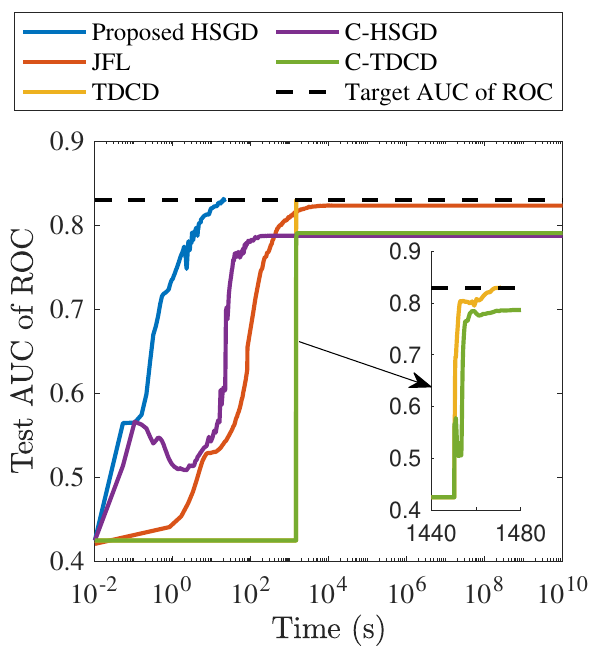}
	\label{Fig4c}}
	\caption{Training performance versus time on three datasets. (a)-(c): Our proposed algorithm consumes less time to achieve target training requirements compared to baselines.  (c): When the size of raw data is large, the proposed algorithm begins model training earlier contributing to saving training time.}
	\label{Baselines_training_time}
\end{figure*}

\begin{figure*}[!t]
	\centering
	\centering
	\subfloat[OrganAMNIST]{\includegraphics[width=0.66\columnwidth]{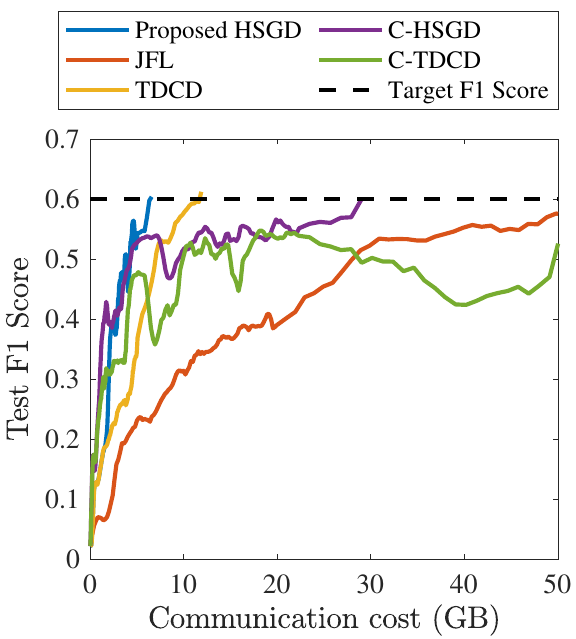}
	\label{Fig5a}}
	\subfloat[ESR]{\includegraphics[width=0.66\columnwidth]{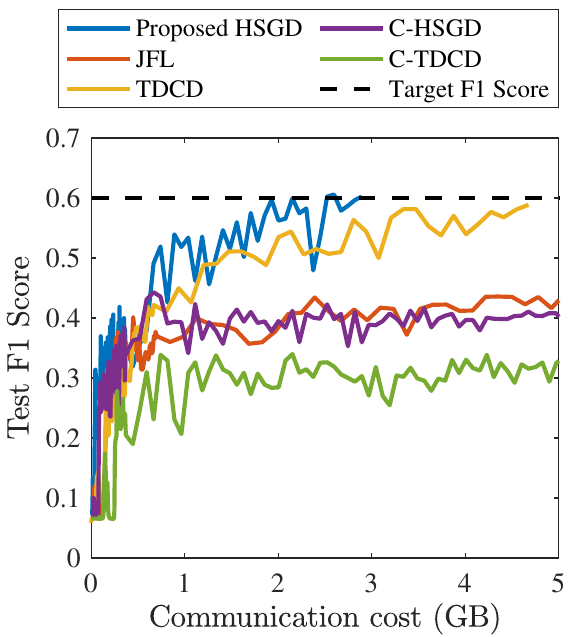}
	\label{Fig5b}}
	\subfloat[MIMIC-III]{\includegraphics[width=0.67\columnwidth]{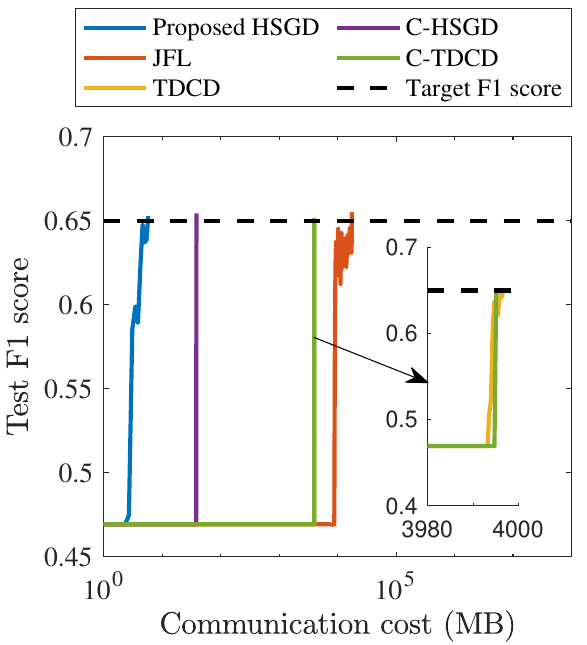}
	\label{Fig5c}}
	\caption{Training performance versus communication cost (each group) on three datasets. The raw data sizes of OrganAMNIST, ESR, and \text{MIMIC-III} are 63~\si{MB}, 7.3~\si{MB}, and 42.3~\si{GB}, respectively.  Our proposed algorithm can reduce the communication cost, and its advantage becomes obvious when the size of raw data becomes larger.} 
	\label{Baselines_communication_cost}
\end{figure*}

\textit{Data split:} We set $M=10$ and $K_m = 3458$ for all experiments on OrganAMNIST. To align with the data distribution in the \text{e-health} domain, we split OrganAMNIST in three steps: i) the dataset is horizontally split into $M=10$ hospital-patient groups following the non-iid data distribution, where each group contains 3000 samples of only 2 labels and 458 samples of other labels; ii) each sample in one hospital-patient group is vertically split into two sub-images with $300$ pixels and $484$ pixels respectively, where one set of sub-images is located at the hospital and the other one is located at the wearable devices; iii) the set of sub-images located at the wearable devices is further partitioned horizontally among $K_m = 3458$ devices, so that each device holds a unique sub-image. For MIMIC-III, we set $M=10$ and $K_m = 1468$. To split MIMIC-III, we first horizontally divide the dataset among $M=10$ hospital-patient groups. The data is non-iid across groups, where each group keeps an imbalanced number of 1468 samples of 2 labels. Then each sample is vertically split into two parts, where each part consists of 36 features for 48 time steps. The rest process is identical to that of OrganAMNIST but $K_m = 1468$. The splitting of ESR is similar to that of MIMIC-III with a few differences. Here, we have $K_m = 920$, and each group consists of 700 samples belonging to 2 specific labels and 220 samples belonging to other labels. Each sample in ESR is further divided into two parts, each containing 89 features.

\begin{table*}[]
	\centering
	\caption{Communication cost for reaching target training requirements, including training loss, test precision, and test recall. A “–” indicates that the target was not reached during training.\label{Table2}}
	\renewcommand{\arraystretch}{1.3}
\begin{tabular}{@{}l|l|l|llllllllllllll@{}}
\toprule
\multirow{2}*{\shortstack{Dataset}} &\multirow{2}*{\shortstack{Metrics}} &\multirow{2}*{\shortstack{Training targets}} &\multicolumn{5}{c}{Communication cost (each group) for reaching targets} \\     
\cline{4-8}~ &{} &{} &{Proposed HSGD} & {JFL}   & {TDCD} & {C-HSGD}   & {C-TDCD}                                            \\ \hline 
\multirow{6}*{\shortstack{OrganAMNIST}} 
\multirow{6}*{\shortstack{}} &\multirow{2}*{\shortstack{Training loss}}  &{Traget = 1.5}  &{2.32~\si{GB}}  & {31.53~\si{GB}}& {5.30~\si{GB}}& \textbf{1.38~\si{GB}} &{3.79~\si{GB}}\\
~ &{}&{Traget = 0.5}   &\textbf{{42.87~\si{GB}}}  &{88.65~\si{GB}} & {31.53~\si{GB}}  & \makecell[c]{--}  & \makecell[c]{--} \\
\cline{2-8}
&\multirow{2}*{\shortstack{Test precision}}  &{Traget = 50\%}  &\textbf{3.08~\si{GB}}   & {19.52~\si{GB}}  & {5.81~\si{GB}}  & {3.75~\si{GB}}  & {4.15~\si{GB}} \\
~ &{} &{Traget = 70\%}   &\textbf{18.41~\si{GB}}  &\makecell[c]{--} & \makecell[c]{--}  & \makecell[c]{--}    & \makecell[c]{--} \\            
\cline{2-8}
\multirow{6}*{\shortstack{}} &\multirow{2}*{\shortstack{Test recall}}   &{Traget = 50\%}  &\textbf{3.07~\si{GB}}   & {23.91~\si{GB}}  & {7.32~\si{GB}}  & {5.13~\si{GB}}  & {10.49~\si{GB}} \\
~ &{} &{Traget = 70\%}   &\textbf{49.17~\si{GB}}  &\makecell[c]{--}  & \makecell[c]{--}   & \makecell[c]{--}  & \makecell[c]{--} \\     
\hline \hline
\multirow{6}*{\shortstack{ESR}} 
\multirow{6}*{\shortstack{}} &\multirow{2}*{\shortstack{Training loss}}  
&{Traget = 1.2}  &\textbf{{0.15~\si{GB}}}   & {0.53~\si{GB}}  &{0.61~\si{GB}}   & {0.61~\si{GB}}   &  \makecell[c]{--}\\
~ &{}&{Traget = 0.8}   &\textbf{{2.52~\si{GB}}} &{\makecell[c]{--}}  & {{4.41~\si{GB}}}   &{\makecell[c]{--}}  & {\makecell[c]{--}} \\ 
\cline{2-8}&\multirow{2}*{\shortstack{Test precision}}  
&{Traget = 40\%}  &\textbf{{0.15~\si{GB}}}   & {{0.19~\si{GB}}}  & \textbf{{0.15~\si{GB}}}  & {{0.23~\si{GB}}}  & {{0.74~\si{GB}}} \\   
~ &{} &{Traget = 60\%}   &\textbf{{1.92~\si{GB}}}  &{\makecell[c]{--}}  & {{5.08~\si{GB}}}   & {\makecell[c]{--}}   & {\makecell[c]{--}} \\            
\cline{2-8}\multirow{6}*{\shortstack{}} &\multirow{2}*{\shortstack{Test recall}}   &{Traget = 40\%}  &\textbf{{0.1~\si{GB}}}   & {{0.29~\si{GB}}}  & {{0.29~\si{GB}}}  & {{0.14~\si{GB}}}  & {{1.34~\si{GB}}} \\
~ &{} &{Traget = 60\%}   &\textbf{{2.52~\si{GB}}}  &{\makecell[c]{--}}  & {{6.28~\si{GB}}}   &{\makecell[c]{--}}  & {\makecell[c]{--}} \\    
\hline \hline
\multirow{4}*{\shortstack{MIMIC-III}} 
\multirow{4}*{\shortstack{}} &\multirow{2}*{\shortstack{Training loss}}  
&{Traget = 0.5}  &\textbf{{0.19~\si{MB}}}   & {{0.63~\si{MB}}}  & {{3990.73~\si{MB}}}  & {{0.23~\si{MB}}}  & {{3990.75~\si{MB}}} \\
~ &{}&{Traget = 0.3}   &\textbf{{9.61~\si{MB}}}  & \makecell[c]{--} & {{3995.89~\si{MB}}}   &  \makecell[c]{--} &  \makecell[c]{--}\\
\cline{2-8}&\multirow{1}*{\shortstack{Test precision}}  
&{Traget = 70\%}  &\textbf{{2.69~\si{MB}}}   & {{34.30~\si{MB}}}  & {{3994.73~\si{MB}}}  & {{38.46~\si{MB}}}  & {{3995.16~\si{MB}}} \\        
\cline{2-8}\multirow{9}*{\shortstack{}} &\multirow{1}*{\shortstack{Test recall}}   
&{Traget = 60\%}  &\textbf{{4.61~\si{MB}}}   & \makecell[c]{--} & {{3994.73~\si{MB}}}  &\makecell[c]{--}& {{40003.58~\si{MB}}} \\
\bottomrule
\end{tabular}
\end{table*}

\textit{Models:} We train CNN models on OrganAMNIST to perform the image classification task, LSTM models on \text{MIMIC-III} to execute the in-hospital mortality prediction task, and LSTM models on ESR to classify epilepsy levels. Each local model trained in our proposed algorithm and all baselines contains three parts: a hospital side model, a device side model, and a combined model. We take the CNN model as an example to illustrate the model structure. We train three CNN models, where two CNN models without fully connected networks are trained based on the data collected by wearable devices and hospitals respectively, and the outputs of these two CNN models are utilized to train the last CNN model which outputs the final predictions (For the detailed structures please refer to Fig.~10 in Appendix~B.).

\textit{3) Training parameters:} For all experiments, we start with the same initial models. In the proposed HSGD algorithm and all baselines, we uniformly sample $\alpha K_m$ devices at random in each communication round, where $\alpha = 0.01$ for OrganAMNIST and $\alpha = 0.02$ for both \text{MIMIC-III} and ESR. Furthermore, in our experiments, each model parameter and gradient is a 32-bit float. For C-TDCD and C-HSGD, the quantization level $b$ is 128 and the compression ratio is $\frac{\log_{2} b}{32}$. In addition, we use an initial learning rate $\eta$ which decays halved per $T_0$ iterations. The initial learning rate of OrganAMNIST is 0.0025 and that of MIMIC-III and ESR is 0.01 unless otherwise specified.

To simulate real-world communication, we consider that wearable devices communicate with edge nodes and hospitals via mobile Internet. The download and upload speeds of mobile Internet are 110~Mbps and 14~Mbps~\cite{Internet_Speeds}, respectively. Moreover, edge nodes, hospitals, and the cloud server communicate via fixed broadband, where download speed is 204~Mbps and upload speed is 74~Mbps~\cite{Internet_Speeds}. According to the communication setting, we can calculate the training time. Taking the proposed HSGD algorithm as an example, for each global round, e-health systems conduct one global aggregation, $\frac{P}{Q}$ times local aggregation, $\frac{P}{Q}$ times intermediate result exchange, and $P$ times local computation. The training time for each global aggregation round is ${t= t_g + \frac{P}{Q}(t_l+t_e)+P\times t_c}$, where $t_g$, $t_l$, and~$t_e$ denote communication time for global aggregation, local aggregation, and intermediate result exchange, respectively. The local computational time $t_c$ is obtained from experiments. Without specified explanation, all comparison algorithms transmit updates every iteration and run one step of SGD between two communications, i.e., $P=Q=1$ for the proposed algorithm.

\begin{table*}[]
	\centering
	\caption{Computational overhead for reaching target AUC of ROC.\label{Table3}}
	\renewcommand{\arraystretch}{1.3}
\begin{tabular}{@{}l|l|lllll@{}}
\toprule
\multirow{2}*{\shortstack{Dataset (The maximum target \\achieved by the worst baseline)}}&\multirow{2}*{\shortstack{Metrics}} &\multicolumn{5}{c}{Computational overhead (each device) for reaching targets} \\     
\cline{3-7}~ &{} &{Proposed HSGD} & {JFL}   & {TDCD} & {C-HSGD}   & {C-TDCD}                                                \\ \hline 
\multirow{2}*{\shortstack{OrganAMNIST\\(Test AUC of ROC = 0.9)}}
\multirow{2}*{\shortstack{}} 
&{Memory consumption}  &\textbf{890~\si{MB}}  & {2600~\si{MB}}& {1950~\si{MB}}& {1530~\si{MB}} &{1734~\si{MB}}\\
&{FLOPs consumption}   &\textbf{{11.57~\si{GFLOPs}}}  & {33.8~\si{GFLOPs}}  & {25.35~\si{GFLOPs}} &{21.60~\si{GFlops}} &{24.48~\si{GFLOPs}}  \\
\hline \hline
\multirow{2}*{\shortstack{ESR\\(Test AUC of ROC = 0.7)}}
\multirow{2}*{\shortstack{}}
&{Memory consumption}  &\textbf{{7.6~\si{MB}}}   & {17.2~\si{MB}}  &{8.8~\si{MB}}   & {16.4~\si{MB}}   &  {29.4~\si{MB}} \\
&{FLOPs consumption}   &\textbf{{1.48~\si{MFLOPs}}} &{3.37~\si{MFLOPs}}& {1.72~\si{MFLOPs}} &{3.21~\si{MFLOPs}}&{5.80~\si{MFLOPs}}\\
\hline \hline
\multirow{2}*{\shortstack{MIMIC-III\\(Test AUC of ROC = 0.78)}}
\multirow{2}*{\shortstack{}} 
&{Memory consumption}  &\textbf{{0.078~\si{MB}}}   & {{0.6~\si{MB}}}  & {{0.11~\si{MB}}}  & {{3~\si{MB}}}  & {{3.28~\si{MB}}} \\
&{FLOPs consumption}  &\textbf{{11661~\si{FLOPs}}}  &{89700~\si{FLOPs}} &{16445~\si{FLOPs}}  & {448500~\si{FLOPs}} &{490360~\si{Flops}} \\
\bottomrule
\end{tabular}
\end{table*}

\begin{table}[]
	\centering
	\caption{Computational time in each round. The local and global aggregation intervals are 1. \label{Table4}}
	\renewcommand{\arraystretch}{1.3}
\begin{tabular}{@{}l|lllll@{}}
\toprule
\multirow{2}*{\shortstack{Dataset}} &\multicolumn{5}{c}{Computational time (\si{s})} \\    
\cline{2-6}~ &{Proposed HSGD} & {JFL}   & {TDCD} & {C-HSGD}   & {C-TDCD}\\                  \hline 
\multirow{1}*{\shortstack{OrganAMNIST}} &{0.06}  & {0.48}& {0.06}& {0.06} &{0.06}\\
\multirow{1}*{\shortstack{ESR}} &{{0.064}}   & {0.1}  &{0.064}   & {0.064}   &  {0.064}\\
\multirow{1}*{\shortstack{MIMIC-III}}   &{0.05}   & {0.8}  & {0.05}  & {0.05}  & {0.05}\\
\bottomrule
\end{tabular}
\end{table}

\begin{figure}[!t]
	\centering
	\subfloat[In each communication round, the communication time remains unchanged, while the computational time is reduced to 0.1 times the experimental time.]{\includegraphics[width=0.47\columnwidth]{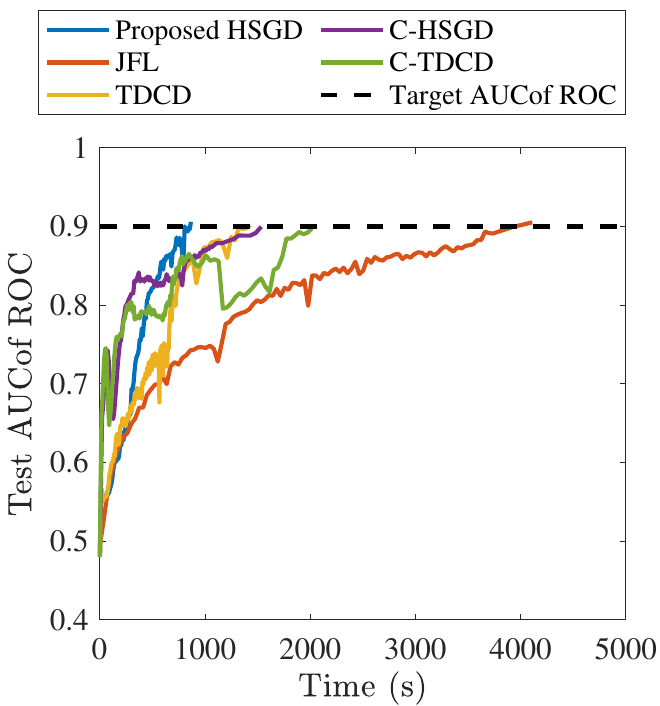}
	\label{Fig6a}}
    \hspace*{\fill}
	\subfloat[In each communication round, the communication time remains unchanged, while the computational time is reduced to 10 times the experimental time.]{\includegraphics[width=0.465\columnwidth]{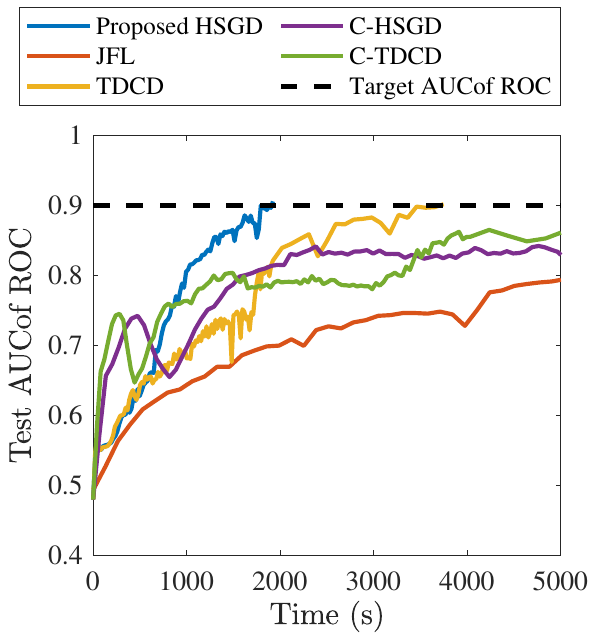}
	\label{Fig6b}}
	\caption{Effect of computational time on training convergence with OrganAMNIST dataset.}
	\label{Baselines_compuational_time}
\end{figure}

\begin{figure*}[!t]
	\centering
	\subfloat[]{\includegraphics[width=0.5\columnwidth]{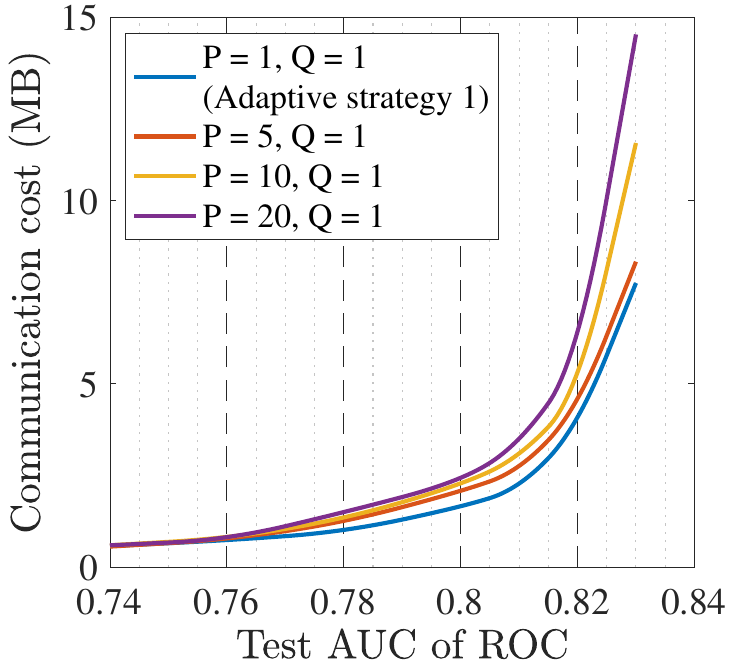}
	\label{Fig7a}}
	\subfloat[]{\includegraphics[width=0.5\columnwidth]{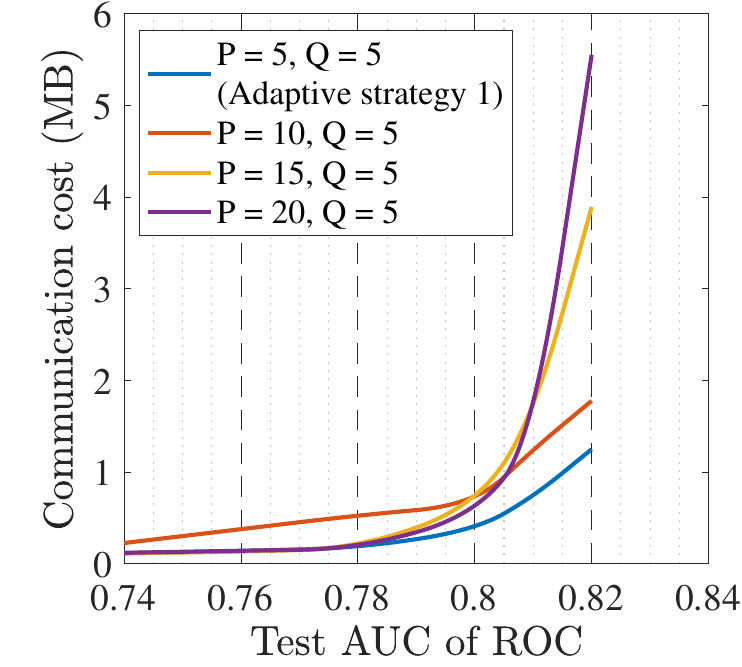}
	\label{Fig7b}}
	\subfloat[]{\includegraphics[width=0.5\columnwidth]{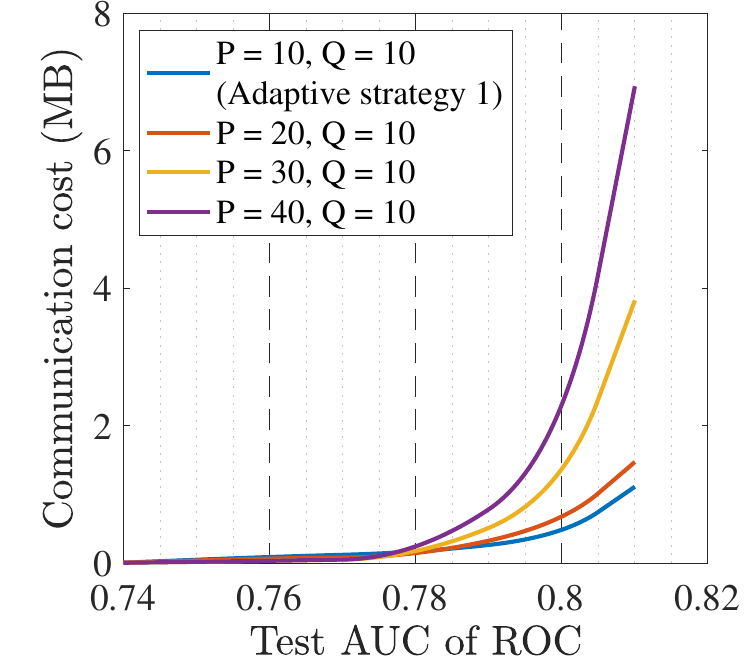}
	\label{Figure6g}}
	\subfloat[]{\includegraphics[width=0.5\columnwidth]{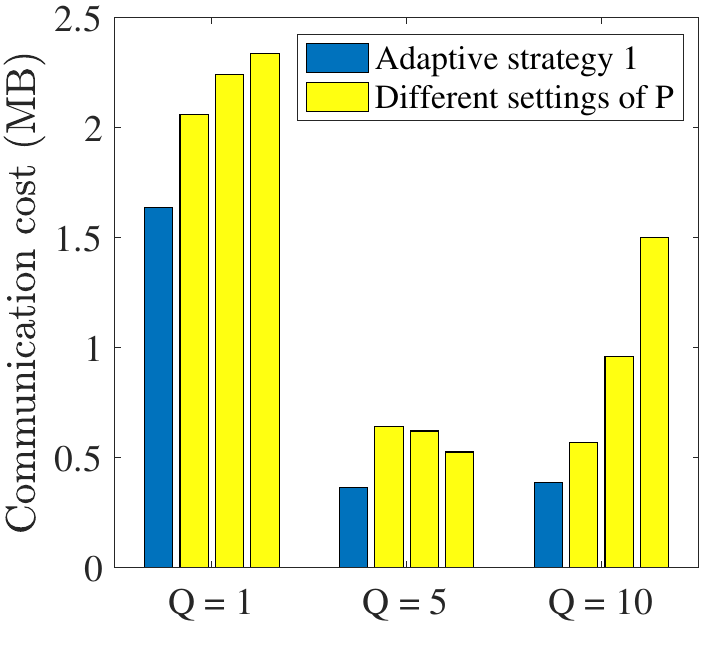}
	\label{Fig7c}}
	\caption{Effect of aggregation intervals on the communication cost of a single group with MIMIC-III dataset. (a)-(c): The communication cost of a single group versus test AUC of ROC with different settings of $P$ and $Q$. (d): The communication cost of a single group to reach 0.8 test AUC of ROC on MIMIC-III. For different $Q$, setting $P=Q$ can always reduce communication cost.}
	\vspace{-0.4cm}
	\label{Proposition1property}
\end{figure*}

\begin{figure}[!t]
		\centering
		\includegraphics[width=0.6\columnwidth]{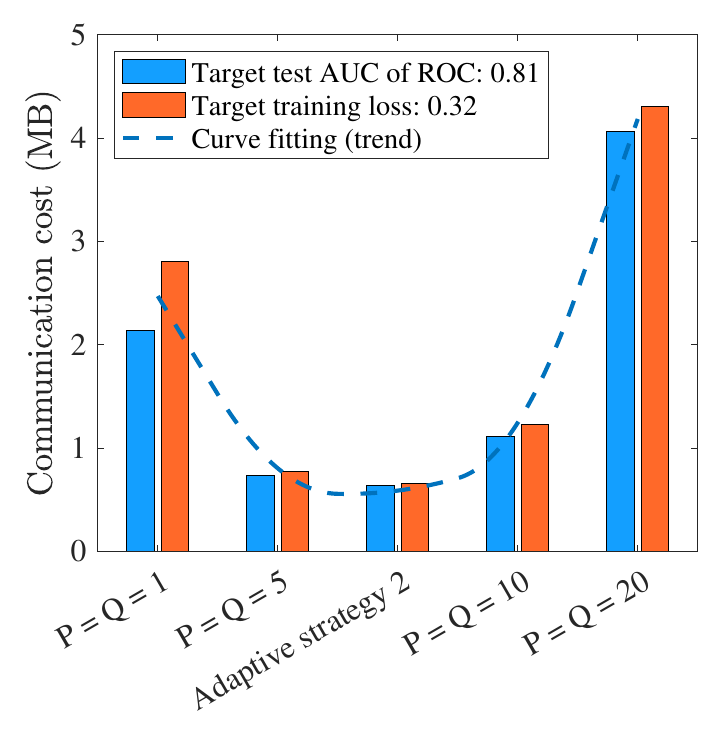}
		\captionsetup{font={scriptsize}}
		\vspace{-0.2cm}
		\caption{The communication cost of one group versus different settings P = Q with MIMIC-III dataset. The transferred data size shows a trend of decrease first and then increase with P and Q growth, and adaptive strategy 2 helps to minimize the communication overhead to achieve the same level of training performance.}
		\label{Proposition2property}
		\vspace{-0.3cm}
\end{figure}

\subsection{Validation of HSGD Algorithm}
\label{sec:Efficiency_of_HSGD_algorithm}

\textit{1) Saving training time:} The training time includes both computational and communication times. We compare the training time of baselines in terms of test Area Under the Curve of the Receiver Operating Characteristic (AUC of ROC). The AUC of ROC is calculated based on a classifier threshold determined by the ROC curve~\cite{Saman2101AUCofROC}. A higher AUC of ROC means a better ability for a classifier to distinguish classes. The results are shown in Fig.~\ref{Baselines_training_time}. To achieve 0.9 test UC of ROC on OrganAMNIST, our proposed method saves over 80\%, 41\%, 56\%, and 62\% training time compared to JFL, TDCD, \text{C-HSGD}, and \text{C-TDCD}, respectively. For JFL, it does not include the local aggregation phase, which causes each hospital to learn a unique model for every wearable device within the same group, resulting in a long training time. Furthermore, since TDCD lacks the global aggregation phase, it overfits local data, resulting in low accuracy when applied to the entire dataset. In addition, when the test AUC of ROC is lower than 0.85/0.8, C-HSGD/C-TDCD outperforms our proposed method. However, our method surpasses C-HSGD/C-TDCD when the test AUC of ROC is over 0.85/0.8. This is because \text{C-HSGD} and C-TDCD save time by compressing the transmitted results, while they also undermine the convergence rate and the achievable training performance as they discard many results. This phenomenon is more remarkable in ESR and \text{MIMIC-III}. In addition, Fig.~\ref{Baselines_training_time}(c) shows that TDCD and C-TDCD start learning at about {1450}~\si{s}. This is because they require transmitting a part of the dataset to convert the three-tier data structure to the two-tier structure. Since the dataset sizes of OrganAMNIST (63~\si{MB}) and ESR (7.3~\si{MB}) are much smaller than that of MIMIC-III (42.3~\si{GB}), the raw data communication time for OrganAMNIST and ESR is tiny. This explains why the aforementioned phenomenon is not remarkably shown in Figs.~\ref{Baselines_training_time}(a) and (b).

\begin{figure*}[!t]
	\centering
	\subfloat[]{\includegraphics[width=0.4\columnwidth]{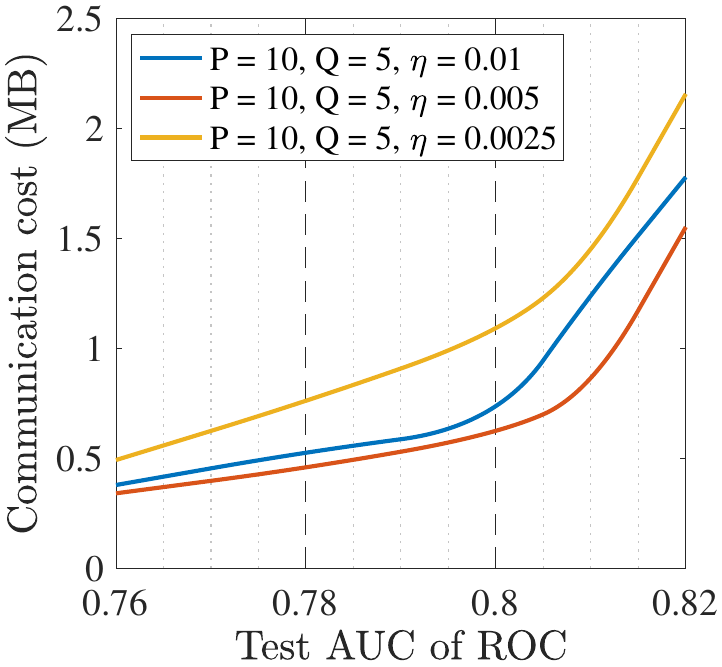}
	\label{Fig9a}}
	\subfloat[]{\includegraphics[width=0.4\columnwidth]{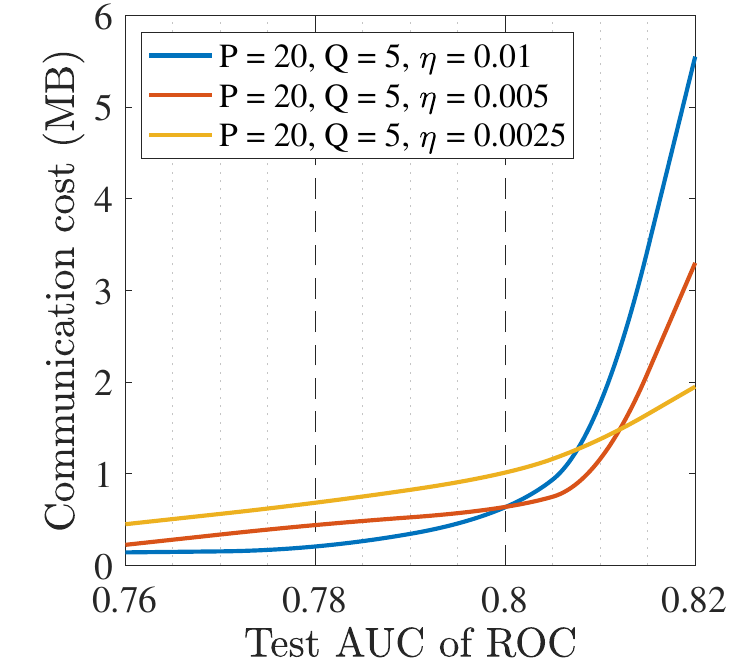}
	\label{Fig9b}}
	\subfloat[]{\includegraphics[width=0.4\columnwidth]{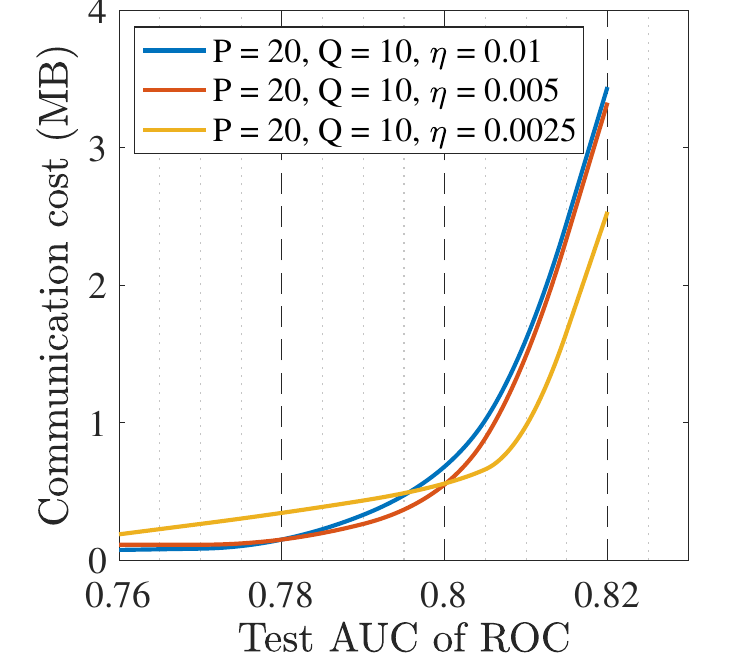}
	\label{Fig9c}}
	\subfloat[]{\includegraphics[width=0.4\columnwidth]{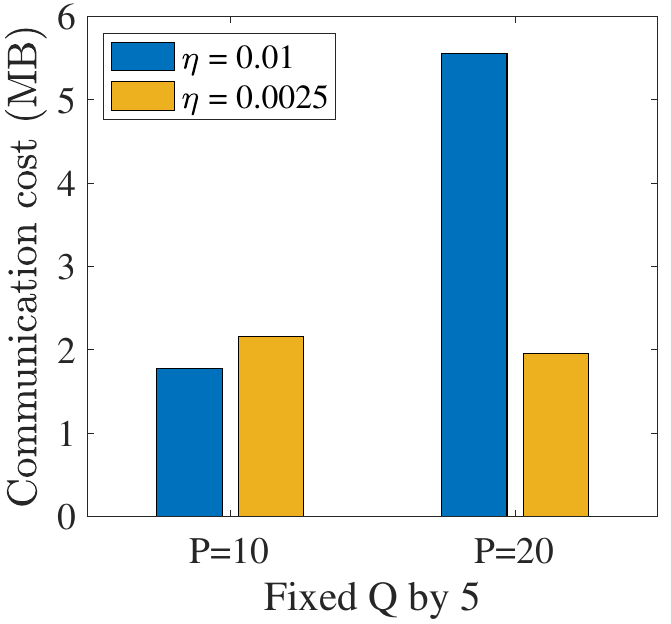}
	\label{Fig9d}}
	\subfloat[]{\includegraphics[width=0.4\columnwidth]{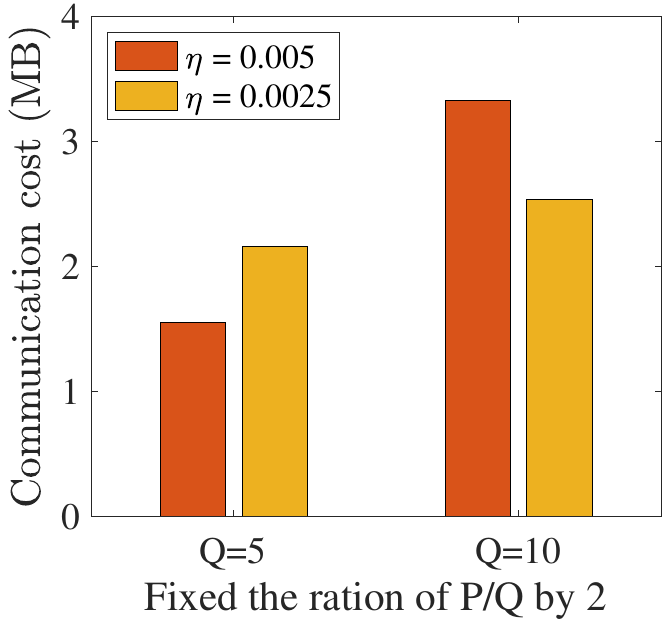}
	\label{Fig9e}}
	\caption{Effect of learning rate on the communication cost of a single group with MIMIC-III dataset. (a)-(c): The communication cost of a single group versus test accuracy with various learning rates. (d)-(e): The communication cost of a single group to reach 0.82 test AUC of ROC on MIMIC-III with various learning rates. For different pairs of $P$ and $Q$, the learning rate should decrease with $Q$ and $P$ to achieve efficient communications.}
	\vspace{-0.3cm}
	\label{Proposition3property}
\end{figure*}

\textit{2) Reducing communication cost:} We then compare the communication cost of the proposed HSGD algorithm and baselines on three datasets. Fig.~\ref{Baselines_communication_cost} shows that the proposed HSGD algorithm can save the communication cost by 45\% and 77\% compared with TDCD and C-HSGD when reaching 0.6 F1 score, while JFL and C-TDCD can not achieve 0.6 F1 score. Note that TDCD and C-TDCD need to combine all data in different groups, which results in high communication cost when the size of raw data is large. Compared with OrganAMNIST and ESR, the communication costs of TDCD and C-TDCD on \text{MIMIC-III} increase significantly because the size of \text{MIMIC-III} (42.3~\si{GB}) is much larger than that of OrganAMNIST (63~\si{MB}) and ESR  (7.3~\si{MB}). In addition, we also compare the communication overhead for reaching different training loss, test precision, and test recall in Table~\ref{Table2}. We can see that the proposed HSGD can substantially reduce communication cost. Note that although \text{C-HSGD} consumes less communication cost to reach 1.5 training loss on OrganAMNIST, it cannot achieve 0.5 training loss, because it compresses models at the expense of training accuracy to save overhead.

\textit{3) Saving computational cost:} 
To evaluate the computation efficiency of the proposed HSGD method, we analyze the memory and FLOPs consumptions. As shown in Table~\ref{Table3}, our proposed method always consumes the least memory and FLOPs to achieve the target test AUC of ROC. For example, to reach 0.9 test AUC of ROC (the maximum target achieved by the worst baseline) on the OrganAMNIST dataset, the proposed method consumes 890~\si{MB} memory and 11.57~\si{GFLOPs}. Compared with baselines, it saves memory consumption by 66\%, 54\%, 42\%, and 48\%, and FLOPs consumption by 66\%, 54\%, 46\%, and 52\%. This is because the memory and FLOPs consumptions are determined by the consumption in each iteration and the total iteration to reach the target requirements. For the proposed method and all baselines, we run the same model, so that their computational consumptions in each iteration are similar (that of C-HSGD and C-TDCD are slightly smaller because they use the compression method), and the total consumptions to achieve the target is mainly determined by the total training iteration. Since our proposed method convergence fast, it reduces the computational cost.

To investigate the impact of computational time on convergence, we list the computational time for each round across all baselines in Table~\ref{Table4}. Here, the computational time is obtain from experiments. The results indicate that the computational times for the proposed HSGD, TDCD, C-HSGD, and C-TDCD are similar because they all train the same local model. On the contrary, JFL requires training multiple local hospital-side models within one group due to the lack of a local aggregation phase, resulting in higher computational time. Subsequently, we vary the computational time to 0.1 times and 10 times the computational time obtained from experiments, respectively. Under these two scenarios, we show the training performance on the OrganAMNIST dataset in Fig.~\ref{Baselines_compuational_time}, where time is the sum of computational and communication times. Comparing Figs.~\ref{Baselines_compuational_time}(a) and (b), we can observe that the increase in computational time can degrade the convergence. In addition, our proposed HSGD converges faster than baselines and the gap in convergence between our method and baselines widens as computational time increases. This is because the proposed method can significantly reduce the training round, thereby saving training time while guaranteeing training accuracy.

\subsection{Proformance of Adaptive Strategies}
\label{sec:Design_principle_of_HSGD_algorithm_validation}
We also conduct experiments to validate the proposed adaptive strategies for our HSGD algorithm. Due to the space limitation, we only show the results of the MIMIC-III dataset.

\textit{1) Adaptive strategy 1 validation}: For the HSGD algorithm, adaptive strategy 1 indicates that the communication cost for achieving a target training requirement can be reduced by setting the same global and local aggregation intervals, i.e., $P=Q$. As we can tell from Figs.~\ref{Proposition1property}(a)-(c), \textit{for different local aggregation Q, adopting adaptive strategy 1 transmits fewer data to achieve the same level test AUC of ROC}. The result in Figs.~\ref{Proposition1property}(d) is more noticeable. For example, when $Q$ is set to be~1, following adaptive strategy 1 can reduce the communication cost by about 20\%, 27\%, and 29\% respectively compared with the other three settings when reaching 0.8 test AUC of ROC. The reason is that when $P$ equals $Q$, the system can conduct global and local aggregations synchronously, eliminating the delay due to asynchronous communication, thereby improving communication efficiency.

\textit{2) Adaptive strategy 2 validation:} We already verify that setting $P=Q$ helps to lessen the communication burden. Now we evaluate the effect of $P$ and $Q$ on communication overhead when satisfying $P=Q$. Fig.~\ref{Proposition2property} shows the transmitted data size of a single group to achieve the required test AUC of ROC and training loss on MIMIC-III. Using the mean values, we fit the trend curve. With $P$ and $Q$ growth, the communication cost shows a trend of first decreasing and then increasing. Adaptive strategy 2 provides the approximate optimal solution of $P$ and $Q$. As we can see from Fig.~\ref{Proposition2property}, \textit{utilizing adaptive strategy 2 can substantially decrease the consumed overhead of communication compared to other configurations.} This is because that frequent communication, such as $P=Q=1$, may result in high accuracy but heavy communication cost. On the contrary, a long communication interval, such as $P=Q=20$, may save communication cost but degrade accuracy. Adaptive strategy 2 can balance the trade-off between training performance and communication cost. The results verify the approximate optimal $P$ and $Q$ given in adaptive strategy 2.

\textit{3) Adaptive strategy 3 validation}:  The choice of learning~$\eta$ has an impact on the HSGD Algorithm when $P$ or $Q$ changes. As described in Adaptive strategy 3, decreasing $\eta$ can improve communication efficiency when $Q$ is fixed and $P$ increases. In addition, when $\frac{P}{Q}$ is fixed and $Q$ grows, reducing $\eta$ can help to save communication cost. Comparing Figs.~\ref{Proposition3property}(a)-(b), we find that setting $\eta =0.01$ can reduce communication cost compared with $\eta =0.0025$ when $P=10,Q=5$, while the result is reversed when $P=20,Q=5$. This result is clearer in Fig.~\ref{Proposition3property}(d). For a $Q$ fixed by 5, when $P$ is set to be 10, using a higher learning rate $\eta$ requires transmitting fewer data to achieve 0.82 test AUC of ROC compared to a lower $\eta$. However, when $P$ increases to 20, decreasing $\eta$ from 0.01 to 0.0025 helps to save communication cost. In addition, from Figs.~\ref{Proposition3property}(a) and (c), the performance of $\eta =0.0025$ overtakes that of $\eta =0.005$ when $Q$ changes from 5 to 10 while $\frac{P}{Q}$ is fixed by 2. Specifically, as shown in Fig.~\ref{Proposition3property}(e), utilizing $\eta =0.0025$ consumes more communication cost than $\eta =0.005$ to obtain 0.82 test AUC of ROC when $\frac{P}{Q}$ is fixed and $Q$ is 5. On the opposite, $\eta =0.0025$ shows advantages in decreasing the size of the transmitted data when $Q$ increases to 10. The reason for the above results is that the communication interval becomes longer with $P$ or $Q$ increases. In this case, a small $\eta$ narrows the cumulative error of models between two communications, requiring fewer communication rounds to reach a certain training requirement. These results verify the effectiveness of the learning rate adjusting method given in adaptive strategy 3.

\section{Conclusion}
\label{sec:Conclusion}
In this work, we have addressed the challenging issue of efficiently combining VFL and HFL to train models for \text{e-health}, where data follows a three-tier horizontal-vertical-horizontal distribution structure. We have proposed a hybrid federated learning framework. In this framework, wearable devices and their corresponding hospitals exchange intermediate results to process vertically partitioned data. Meanwhile, global aggregation is performed on the cloud server for handling horizontally partitioned data. Moreover, to improve communication efficiency, local aggregation is applied across wearable devices. According to the hybrid federated learning framework, we have developed a HSGD algorithm for model training. We also have analyzed the convergence upper bound of the proposed HSGD algorithm. Using this theoretical result, adaptive strategies related to parameter tuning have been proposed to further alleviate the communication burden while achieving desirable training results. Experiments validate that the proposed HSGD algorithm can save training time and communication cost compared to existing approaches. Moreover, experiment results verify the effectiveness of adaptive strategies in reducing the transmitted data. For future work, we will investigate how to select devices to enhance training efficiency, detect adversarial attacks in federated learning, and study security schemes against attacks.

\bibliographystyle{IEEEtran}
\bibliography{Bib/Ref}

\clearpage
\appendices

\onecolumn
\section*{Supplementary material}
\section{}
\label{Appendix_A}
\subsection{Proof of Lemmas}
In this section, we provide the proofs of lemmas associated with Theorem~1.

To facilitate the analysis, we first define an auxiliary local vector $\bm \phi_{i,m}^t$, which represents the local view of the global model on the hospital at each hospital-patient group, to calculate the partial derivative,
\begin{equation}
\begin{aligned}
{}&{\bm \phi_{0,m}^t= \bm \phi_{1,m}^t:= [({\bm\theta}_{0,m}^{t}})^{\mathrm{T}}, ({\bm\theta}_{1,m}^{t})^{\mathrm{T}}, ({\bm\theta}_{2,m}^{t_0^{\lambda}})^{\mathrm{T}}]^{\mathrm{T}}, \lambda=\lfloor \frac{t-t_0}{Q}\rfloor,
\label{auxiliary_local_vector}	
\end{aligned}
\end{equation}
where $t_0$ denotes the last iteration when the server aggregates the global model, $t_0^{\lambda}$ represents the last iteration when hospitals and edge nodes agree on a new mini-batch. We define $\frac{P}{Q}=\Lambda$ and $t_0^{\lambda}=t_0+\lambda Q$, where $\Lambda$ is a positive integer. Then each global aggregation interval $[t_0,t_0+P-1]$ can be divided into multiple local aggregation intervals, i.e., $[t_0^{\lambda},t_0^{\lambda+1}-1], \lambda=\{0,1,..,\Lambda-1\}$ with $t_0^{0}=t_0$. During each local aggregation interval, we use the same mini-batch to train local models.

Since our algorithm introduces an additional local aggregation among wearable devices within each group to obtain ${\bm\theta}_{2,m}^t$, we can directly obtain ${\bm\theta}_{2,m,n}^{t}$ on device $n$ in group $m$ instead of ${\bm\theta}_{2,m}^t$. For simplicity, we first define an auxiliary local vector $\bm\phi_{2,m,n}^t$, to represent the local view of the global model on each device at each hospital-patient group,
\begin{equation}
\begin{aligned}
\bm \phi_{2,m,n}^t:= [({\bm\theta}_{0,m}^{t_0^{\lambda}})^{\mathrm{T}}, ({\bm\theta}_{1,m}^{t_0^{\lambda}})^{\mathrm{T}}, ({\bm\theta}_{2,m,n}^{t})^{\mathrm{T}}]^{\mathrm{T}}.
\label{auxiliary_local_vector_2}	
\end{aligned}
\end{equation}

With slight abuse of notation, the partial derivative $\bm g_{i,m}$ can be represented by 
\begin{equation}
\begin{aligned}
&\bm g_{i,m} \left(\bm \phi_{i,m}^t\right):=\bm g_{i,m}({\bm\theta}_{0,m}^{t},{\bm\theta}_{1,m}^{t}, {\bm\theta}_{2,m}^{t_0^{\lambda}}; {\xi}_m^{t_0^{\lambda}}), \forall i=0,1,
\label{partial_derivative_1}	
\end{aligned}
\end{equation}
\begin{equation}
\begin{aligned}
&\bm g_{2,m,n} \left(\bm \phi_{2,m,n}^t\right):=\bm g_{2,m,n}({\bm\theta}_{0,m}^{t_0^{\lambda}}, {\bm\theta}_{1,m}^{t_0^{\lambda}}, {\bm\theta}_{2,m,n}^{t}; {\xi}_{m,n}^{t_0^{\lambda}}).
\label{partial_derivative_2}	
\end{aligned}
\end{equation}

Since our proposed method reuses each mini-batch for $Q$ iterations during each local aggregation interval, the local stochastic partial derivatives are not unbiased during interval $[t_0^{\lambda},t_0^{\lambda+1}], \lambda=\{0,1,..,\Lambda-1\}$. According to the conditional expectation on Assumption~2, the gradients calculated at $t_0^{\lambda}$ satisfies 
\begin{equation}
\begin{aligned}
\mathbb{E}[\bm g_{i,m}(\bm\phi_{i,m}^{t_0^{\lambda}})\mid \bm \phi_{i,m}^{t_0^{\lambda}}]= \nabla_{(i)} F(\bm \phi_{i,m}^{t_0^{\lambda}}), \forall i=0,1,
\label{Appendix_EqAssumption2_1}	
\end{aligned}
\end{equation}
\begin{equation}
\begin{aligned}
\mathbb{E}[\bm g_{2,m,n}(\bm\phi_{2,m,n}^{t_0^{\lambda}})\mid \bm \phi_{2,m,n}^{t_0^{\lambda}}]= \nabla_{(2)} F(\bm \phi_{2,m,n}^{t_0^{\lambda}}),
\label{Appendix_EqAssumption2_2}	
\end{aligned}
\end{equation}
\begin{equation}
\begin{aligned}
\mathbb{E}\left[\left\Vert \bm g_{i,m}(\bm\phi_{i,m}^{t_0^{\lambda}})-\nabla_{(i)} F(\bm\phi_{i,m}^{t_0^{\lambda}})\right\Vert ^2 \mid \bm\phi_{i,m}^{t_0^{\lambda}}\right]\leq \delta^2, \forall i=0,1,
\label{Appendix_EqAssumption2_3}	
\end{aligned}
\end{equation}
\begin{equation}
\begin{aligned}
\mathbb{E}\left[\left\Vert \bm g_{2,m,n}(\bm\phi_{2,m,n}^{t_0^{\lambda}})-\nabla_{(2)} F(\bm\phi_{2,m,n}^{t_0^{\lambda}})\right\Vert ^2 \mid \bm\phi_{2,m,n}^{t_0^{\lambda}}\right]\leq \delta^2.
\label{Appendix_EqAssumption2_4}	
\end{aligned}
\end{equation}

For notation brevity, we define $\Phi^{t_0^{\lambda}}=\{{\bm\phi_{0,m}^{t_0^{\lambda}},\bm\phi_{2,m}^{t_0^{\lambda}},\bm\phi_{2,m,n}^{t_0^{\lambda}}}\}$ and ease the conditional expectation as the following
\begin{equation}
\begin{aligned}
\mathbb{E}^{t_0^{\lambda}}:=\mathbb{E}[\ \bm \cdot\mid \Phi^{t_0^{\lambda}}].
\label{Appendix_notation_brevity}	
\end{aligned}
\end{equation}

Recall that we define the update process of the global model as
\begin{equation}
\begin{aligned}
\tilde{\bm\theta}^{t+1} = \tilde{\bm\theta}^t-\eta \bm G^t.
\label{Appendix_GlobalUpdate}	
\end{aligned}
\end{equation}

Then, the following lemmas are given to prove Theorem~1.
\begin{lemma}\label{Lemma1} 
\begin{equation}
\begin{aligned}
\mathbb{E}\left\Vert \bm g_{i,m}(\bm \phi_{i,m}^{t_0^{\lambda}})\right\Vert^2 \leq \delta^2+\mathbb{E}\left\Vert \nabla_{(i)}F({\bm \phi_{i,m}^{t_0^{\lambda}})}\right\Vert^2, \forall i={0,1},
\label{EqLemma1}	
\end{aligned}
\end{equation}
where $\mathbb{E}$ is the total expectation.
\end{lemma}

\begin{proof}
\begin{equation}
\begin{aligned}
\mathbb{E}\left\Vert \bm g_{i,m}(\bm \phi_{i,m}^{t_0^{\lambda}})\right\Vert^2 &= \mathbb{E}\left\Vert \bm g_{i,m}(\bm \phi_{i,m}^{t_0^{\lambda}})-\mathbb{E}\left[\bm g_{i,m}(\bm \phi_{i,m}^{t_0^{\lambda}})\right]+\mathbb{E}\left[\bm g_{i,m}(\bm \phi_{i,m}^{t_0^{\lambda}})\right]\right\Vert^2\\
&= \mathbb{E}\left[\mathbb{E}^{t_0^{\lambda}}\left\Vert \bm g_{i,m}(\bm \phi_{i,m}^{t_0^{\lambda}})-\mathbb{E}\left[\bm g_{i,m}(\bm \phi_{i,m}^{t_0^{\lambda}})\right]+\mathbb{E}\left[\bm g_{i,m}(\bm \phi_{i,m}^{t_0^{\lambda}})\right]\right\Vert^2\right]\\
&\overset{(a)}= \mathbb{E}\left[\mathbb{E}^{t_0^{\lambda}}\left\Vert \bm g_{i,m}(\bm \phi_{i,m}^{t_0^{\lambda}})-\nabla_{(i)}F({\bm \phi_{i,m}^{t_0^{\lambda}})}\right\Vert^2\right]+\mathbb{E}\left[\left\Vert\mathbb{E}^{t_0^{\lambda}}\left[\bm g_{i,m}(\bm \phi_{i,m}^{t_0^{\lambda}})\right]\right\Vert^2\right]\\
&\overset{(b)}\leq \delta^2+\mathbb{E}\left\Vert \nabla_{(i)}F({\bm \phi_{i,m}^{t_0^{\lambda}})}\right\Vert^2.
\label{ProofLemma1Eq1}
\end{aligned}
\end{equation}
where equality in $(a)$ is because of~\eqref{Appendix_EqAssumption2_1} in Assumption~2, and inequality in $(b)$ follows~\eqref{Appendix_EqAssumption2_3} in Assumption~2.
\end{proof}

\begin{lemma}\label{Lemma2} 
\begin{equation}
\begin{aligned}
\mathbb{E}\left\Vert \bm g_{2,m,n}(\bm \phi_{2,m,n}^{t_0^{\lambda}})\right\Vert^2 \leq \delta^2+\mathbb{E}\left\Vert \nabla_{(2)}F({\bm \phi_{2,m,n}^{t_0^{\lambda}})}\right\Vert^2,
\label{EqLemma2}	
\end{aligned}
\end{equation}
where $\mathbb{E}$ is the total expectation.
\end{lemma}

\begin{proof}
The proof of Lemma~\ref{Lemma2} follows the same idea of that of Lemma~\ref{Lemma1}. To save space, we do not show the details here.
\end{proof}

\begin{lemma}\label{Lemma3} 
When $t_0^{\lambda}\leq t \leq t_0^{\lambda+1}-1$ and $\eta \leq \frac{1}{2Q\rho}$,
\begin{equation}
\begin{aligned}
\sum_{t=t_0^{\lambda}}^{t_0^{\lambda+1}-1}\mathbb{E}^{t_0^{\lambda}}\left\Vert\ \bm g_{i,m}(\bm \phi_{i,m}^{t}) - \bm g_{i,m}(\bm \phi_{i,m}^{t_0^{\lambda}})\right\Vert^2 \leq 4Q^2(1+Q)\rho^2\eta^2\left\Vert \bm g_{i,m}(\bm \phi_{i,m}^{t_0^{\lambda}})\right\Vert^2, \forall i={0,1}.
\label{EqLemma3}	
\end{aligned}
\end{equation}
\end{lemma}

\begin{proof}

\begin{equation}
\begin{aligned}
{}&\mathbb{E}^{t_0^{\lambda}}\left\Vert \bm g_{i,m}(\bm \phi_{i,m}^{t+1})-\bm g_{i,m}(\bm \phi_{i,m}^{t_0^{\lambda}})\right\Vert^2\\
{}&=\mathbb{E}^{t_0^{\lambda}}\left\Vert \bm g_{i,m}(\bm \phi_{i,m}^{t+1})-\bm g_{i,m}(\bm \phi_{i,m}^{t})+\bm g_{i,m}(\bm \phi_{i,m}^{t})-\bm g_{i,m}(\bm \phi_{i,m}^{t_0^{\lambda}})\right\Vert^2\\
{}&\leq\mathbb{E}^{t_0^{\lambda}}\left[\left(1+q\right)\left\Vert \bm g_{i,m}(\bm \phi_{i,m}^{t+1})-\bm g_{i,m}(\bm \phi_{i,m}^{t})\right\Vert^2+\left(1+\frac{1}{q}\right)\left\Vert \bm g_{i,m}(\bm \phi_{i,m}^{t})-\bm g_{i,m}(\bm \phi_{i,m}^{t_0^{\lambda}})\right\Vert^2\right]\\
{}&\overset{(a)}\leq\mathbb{E}^{t_0^{\lambda}}\left[\left(1+q\right)\rho_i^2\left\Vert \bm \phi_{i,m}^{t+1}-\bm \phi_{i,m}^{t}\right\Vert^2+\left(1+\frac{1}{q}\right)\left\Vert \bm g_{i,m}(\bm \phi_{i,m}^{t})-\bm g_{i,m}(\bm \phi_{i,m}^{t_0^{\lambda}})\right\Vert^2\right]\\
{}&\leq\mathbb{E}^{t_0^{\lambda}}\left[\left(1+q\right)\rho_i^2\eta^2\left\Vert \bm g_{i,m}(\bm \phi_{i,m}^{t})\right\Vert^2+\left(1+\frac{1}{q}\right)\left\Vert \bm g_{i,m}(\bm \phi_{i,m}^{t})-\bm g_{i,m}(\bm \phi_{i,m}^{t_0^{\lambda}})\right\Vert^2\right]\\
{}&\leq\mathbb{E}^{t_0^{\lambda}}\left[\left(1+q\right)\rho_i^2\eta^2\left\Vert \bm g_{i,m}(\bm \phi_{i,m}^{t})-\bm g_{i,m}(\bm \phi_{i,m}^{t_0^{\lambda}})+\bm g_{i,m}(\bm \phi_{i,m}^{t_0^{\lambda}})\right\Vert^2+\left(1+\frac{1}{q}\right)\left\Vert \bm g_{i,m}(\bm \phi_{i,m}^{t})-\bm g_{i,m}(\bm \phi_{i,m}^{t_0^{\lambda}})\right\Vert^2\right]\\
{}&\leq\mathbb{E}^{t_0^{\lambda}}\left[\left(1+q\right)\rho_i^2\eta^2\left\Vert \bm g_{i,m}(\bm \phi_{i,m}^{t})-\bm g_{i,m}(\bm \phi_{i,m}^{t_0^{\lambda}})+\bm g_{i,m}(\bm \phi_{i,m}^{t_0^{\lambda}})\right\Vert^2\right]+\mathbb{E}^{t_0^{\lambda}}\left[\left(1+\frac{1}{q}\right)\left\Vert \bm g_{i,m}(\bm \phi_{i,m}^{t})-\bm g_{i,m}(\bm \phi_{i,m}^{t_0^{\lambda}})\right\Vert^2\right]\\
{}&\leq\mathbb{E}^{t_0^{\lambda}}\left[2(1+q)\rho_i^2\eta^2\left\Vert \bm g_{i,m}(\bm \phi_{i,m}^{t})-\bm g_{i,m}(\bm \phi_{i,m}^{t_0^{\lambda}}))\right\Vert^2\right]+2(1+q)\rho_i^2\eta^2\left\Vert \bm g_{i,m}(\bm \phi_{i,m}^{t_0^{\lambda}})\right\Vert^2\\
&\ \ +\mathbb{E}^{t_0^{\lambda}}\left[\left(1+\frac{1}{q}\right)\left\Vert \bm g_{i,m}(\bm \phi_{i,m}^{t})-\bm g_{i,m}(\bm \phi_{i,m}^{t_0^{\lambda}})\right\Vert^2\right]\\
{}&\leq\left(2\left(1+q\right)\rho_i^2\eta^2+\left(1+\frac{1}{q}\right)\right)\mathbb{E}^{t_0^{\lambda}}\left\Vert \bm g_{i,m}(\bm \phi_{i,m}^{t})-\bm g_{i,m}(\bm \phi_{i,m}^{t_0^{\lambda}})\right\Vert^2+2(1+q)\rho_i^2\eta^2\left\Vert \bm g_{i,m}(\bm \phi_{i,m}^{t_0^{\lambda}})\right\Vert^2\\
{}&\overset{(b)}\leq\left(2\left(1+q\right)\rho^2\eta^2+\left(1+\frac{1}{q}\right)\right)\mathbb{E}^{t_0^{\lambda}}\left\Vert \bm g_{i,m}(\bm \phi_{i,m}^{t})-\bm g_{i,m}(\bm \phi_{i,m}^{t_0^{\lambda}})\right\Vert^2+2\left(1+q\right)\rho^2\eta^2\left\Vert \bm g_{i,m}(\bm \phi_{i,m}^{t_0^{\lambda}})\right\Vert^2.
\label{ProofLemma3Eq1}
\end{aligned}
\end{equation}

Here inequality in $(a)$ is based on Assumption~1, and $(b)$ is because $\rho_i\leq \rho$. On setting $q=Q$ and $\eta\leq\frac{1}{2Q\rho}$ in the first term of~\eqref{ProofLemma3Eq1}, we obtain the following recurrence relation
\begin{equation}
\begin{aligned}
{}&\mathbb{E}^{t_0^{\lambda}}\left\Vert \bm g_{i,m}(\bm \phi_{i,m}^{t+1})-\bm g_{i,m}(\bm \phi_{i,m}^{t_0^{\lambda}})\right\Vert^2\\
{}&\leq\left(\frac{1+Q}{2Q^2}+\left(1+\frac{1}{Q}\right)\right)\mathbb{E}^{t_0^{\lambda}}\left\Vert \bm g_{i,m}(\bm \phi_{i,m}^{t})-\bm g_{i,m}(\bm \phi_{i,m}^{t_0^{\lambda}})\right\Vert^2+2\left(1+Q\right)\rho^2\eta^2\left\Vert \bm g_{i,m}(\bm \phi_{i,m}^{t_0^{\lambda}})\right\Vert^2\\
{}&\leq\left(1+\frac{2}{Q}\right)\mathbb{E}^{t_0^{\lambda}}\left\Vert \bm g_{i,m}(\bm \phi_{i,m}^{t})-\bm g_{i,m}(\bm \phi_{i,m}^{t_0^{\lambda}})\right\Vert^2+2\left(1+Q\right)\rho^2\eta^2\left\Vert \bm g_{i,m}(\bm \phi_{i,m}^{t_0^{\lambda}})\right\Vert^2.
\label{ProofLemma3Eq2}
\end{aligned}
\end{equation}
Using this recurrence relation, we can obtain 
\begin{equation}
\begin{aligned}
{}&\mathbb{E}^{t_0^{\lambda}}\left\Vert \bm g_{i,m}(\bm \phi_{i,m}^{t})-\bm g_{i,m}(\bm \phi_{i,m}^{t_0^{\lambda}})\right\Vert^2\\
{}&\leq \left(1+\frac{2}{Q}\right)^{t-t_0^{\lambda}}\mathbb{E}^{t_0^{\lambda}}\left\Vert \bm g_{i,m}(\bm \phi_{i,m}^{t_0^{\lambda}})-\bm g_{i,m}(\bm \phi_{i,m}^{t_0^{\lambda}})\right\Vert^2+2\left(1+Q\right)\rho^2\eta^2\left\Vert \bm g_{i,m}(\bm \phi_{i,m}^{t_0^{\lambda}})\right\Vert^2\sum_{p=0}^{t-t_0^{\lambda}-1}\left(1+\frac{2}{Q}\right)^p \\
{}&\leq 2\left(1+Q\right)\rho^2\eta^2\left\Vert \bm g_{i,m}(\bm \phi_{i,m}^{t_0^{\lambda}})\right\Vert^2\sum_{p=0}^{t-t_0^{\lambda}-1}\left(1+\frac{2}{Q}\right)^p\\
{}&\leq 2\left(1+Q\right)\rho^2\eta^2\left\Vert \bm g_{i,m}(\bm \phi_{i,m}^{t_0^{\lambda}})\right\Vert^2 \frac{\left(1+\frac{2}{Q}\right)^{t-t_0}-1}{\left(1+\frac{2}{Q}\right)-1}.
\label{ProofLemma3Eq3}
\end{aligned}
\end{equation}
Then summing $\mathbb{E}^{t_0^{\lambda}}\left\Vert \bm g_{i,m}(\bm \phi_{i,m}^{t})-\bm g_{i,m}(\bm \phi_{i,m}^{t_0^{\lambda}})\right\Vert^2$ over the set of $[t_0^{\lambda},t_0^{\lambda+1}-1]$, we have
\begin{equation}
\begin{aligned}
&\sum_{t=t_0^{\lambda}}^{t_0^{\lambda+1}-1}\mathbb{E}^{t_0^{\lambda}}\left\Vert \bm g_{i,m}(\bm \phi_{i,m}^{t})-\bm g_{i,m}(\bm \phi_{i,m}^{t_0^{\lambda}})\right\Vert^2\\
&\leq \sum_{t=t_0^{\lambda}}^{t_0^{\lambda+1}-1} 2\left(1+Q\right)\rho^2\eta^2\left\Vert \bm g_{i,m}(\bm \phi_{i,m}^{t_0^{\lambda}})\right\Vert^2 \frac{\left(1+\frac{2}{Q}\right)^{t-t_0}-1}{\left(1+\frac{2}{Q}\right)-1}\\
&\leq 2\left(1+Q\right)\rho^2\eta^2\left\Vert \bm g_{i,m}(\bm \phi_{i,m}^{t_0^{\lambda}})\right\Vert^2 \sum_{t=t_0^{\lambda}}^{t_0^{\lambda+1}-1} \frac{\left(1+\frac{2}{Q}\right)^{t-t_0}-1}{\left(1+\frac{2}{Q}\right)-1}\\
&\leq \frac{2\left(1+Q\right)\rho^2\eta^2\left\Vert \bm g_{i,m}(\bm \phi_{i,m}^{t_0^{\lambda}})\right\Vert^2}{\left(1+\frac{2}{Q}\right)-1} \sum_{t=t_0^{\lambda}}^{t_0^{\lambda+1}-1} \left(\left(1+\frac{2}{Q}\right)^{t-t_0}-1\right)\\
&\leq \frac{2\left(1+Q\right)\rho^2\eta^2\left\Vert \bm g_{i,m}(\bm \phi_{i,m}^{t_0^{\lambda}})\right\Vert^2}{\left(1+\frac{2}{Q}\right)-1}  \left(\frac{\left(1+\frac{2}{Q}\right)^Q}{\left(1+\frac{2}{Q}\right)-1}-Q\right)\\
&\leq Q(1+Q)\rho^2\eta^2\left\Vert \bm g_{i,m}(\bm \phi_{i,m}^{t_0^{\lambda}})\right\Vert^2 \left(\frac{Q\left(1+\frac{2}{Q}\right)^Q}{2}-Q\right)\\
&\leq Q^2(1+Q)\rho^2\eta^2\left\Vert \bm g_{i,m}(\bm \phi_{i,m}^{t_0^{\lambda}})\right\Vert^2 \left(\frac{\left(1+\frac{2}{Q}\right)^Q}{2}-1\right)\\
&\overset{(a)}\leq Q^2(1+Q)\rho^2\eta^2\left\Vert \bm g_{i,m}(\bm \phi_{i,m}^{t_0^{\lambda}})\right\Vert^2 \left(\frac{e^2-1}{2}-1\right)\\
&\leq 4Q^2(1+Q)\rho^2\eta^2\left\Vert \bm g_{i,m}(\bm \phi_{i,m}^{t_0^{\lambda}})\right\Vert^2.
\label{ProofLemma3Eq4}
\end{aligned}
\end{equation}
Here we upper bound the terms inside the parenthesis of $(a)$ by 4.
\end{proof}

\begin{lemma}\label{Lemma4} 
When  $t_0^{\lambda}\leq t \leq t_0^{\lambda+1}-1$ and $\eta \leq \frac{1}{2Q\rho}$
\begin{equation}
\begin{aligned}
\sum_{t=t_0^{\lambda}}^{t_0^{\lambda+1}-1}\mathbb{E}^{t_0^{\lambda}}\left\Vert\ \bm g_{2,m,n}(\bm \phi_{2,m,n}^{t}) - \bm g_{2,m,n}(\bm \phi_{2,m,n}^{t_0^{\lambda}})\right\Vert^2 \leq 4Q^2(1+Q)\rho^2\eta^2\left\Vert \bm g_{2,m,n}(\bm \phi_{2,m,n}^{t_0^{\lambda}})\right\Vert^2.
\label{EqLemma4}	
\end{aligned}
\end{equation}
\end{lemma}

\begin{proof}
The proof of Lemma~\ref{Lemma4} is similar to that of Lemma~\ref{Lemma3}. For brevity, we omit the details herein.
\end{proof}

\begin{lemma}\label{Lemma5} 
\begin{equation}
\begin{aligned}
\mathbb{E}^{t_0^{\lambda}}\left\Vert \eta \sum_{t=t_0^{\lambda}}^{t_0^{\lambda+1}-1}\bm  g_{(i)}^t\right\Vert^2 \leq 2\eta^2Q \left(4Q^2(1+Q)\rho^2\eta^2+Q\right)\frac{1}{K}\sum_{m=1}^{M}K_m \left\Vert \bm g_{i,m}(\bm \phi_{i,m}^{t_0^{\lambda}})\right\Vert^2.
\label{EqLemma5}	
\end{aligned}
\end{equation}
\end{lemma}

\begin{proof}
\begin{equation}
\begin{aligned}
&\mathbb{E}^{t_0^\lambda}\left\Vert \eta \sum_{t=t_0^{\lambda}}^{t_0^{\lambda+1}-1}\bm G_{(i)}^t\right\Vert^2\\
&\overset{(a)} \leq \eta^2Q \sum_{t=t_0^{\lambda}}^{t_0^{\lambda+1}-1}\mathbb{E}^{t_0^\lambda} \left\Vert\bm G_{(i)}^t\right\Vert^2\\
&\leq \eta^2Q \sum_{t=t_0^{\lambda}}^{t_0^{\lambda+1}-1}\mathbb{E}^{t_0^\lambda} \left\Vert\bm G_{(i)}^t -\bm G_{(i)}^{t_0^\lambda}+\bm G_{(i)}^{t_0^\lambda}\right\Vert^2\\
&\leq 2\eta^2Q \sum_{t=t_0^{\lambda}}^{t_0^{\lambda+1}-1}\mathbb{E}^{t_0^\lambda} \left\Vert\bm G_{(i)}^t -\bm G_{(i)}^{t_0^\lambda}\right\Vert^2 +2\eta^2Q\sum_{t=t_0^{\lambda}}^{t_0^{\lambda+1}-1} \left\Vert\bm G_{(i)}^{t_0^\lambda}\right\Vert^2\\
&\overset{(b)}\leq 2\eta^2Q \sum_{t=t_0^{\lambda}}^{t_0^{\lambda+1}-1}\frac{1}{K}\sum_{m=1}^{M}K_m\mathbb{E}^{t_0^\lambda} \left\Vert \bm g_{i,m}(\bm \phi_{i,m}^{t})-\bm g_{i,m}(\bm \phi_{i,m}^{t_0^\lambda})\right\Vert^2 +2\eta^2Q\sum_{t=t_0^{\lambda}}^{t_0^{\lambda+1}-1}\frac{1}{K}\sum_{m=1}^{M}K_m \left\Vert\bm g_{i,m}(\bm \phi_{i,m}^{t_0^\lambda})\right\Vert^2\\
&\overset{(c)}\leq 2\eta^2Q \left(4Q^2(1+Q)\rho^2\eta^2\right)\frac{1}{K}\sum_{m=1}^{M}K_m \left\Vert \bm g_{i,m}(\bm \phi_{i,m}^{t_0^\lambda})\right\Vert^2 +2\eta^2Q \sum_{t=t_0^{\lambda}}^{t_0^{\lambda+1}-1}\frac{1}{K}\sum_{m=1}^{M}K_m \left\Vert\bm g_{i,m}(\bm \phi_{i,m}^{t_0^\lambda})\right\Vert^2\\
&\leq 2\eta^2Q \left(4Q^2(1+Q)\rho^2\eta^2\right)\frac{1}{K}\sum_{m=1}^{M}K_m \left\Vert \bm g_{i,m}(\bm \phi_{i,m}^{t_0^\lambda})\right\Vert^2 +2\eta^2Q^2\frac{1}{K}\sum_{m=1}^{M}K_m \left\Vert\bm g_{i,m}(\bm \phi_{i,m}^{t_0^\lambda})\right\Vert^2\\
&\leq  2\eta^2Q \left(4Q^2(1+Q)\rho^2\eta^2+Q\right)\frac{1}{K}\sum_{m=1}^{M}K_m \left\Vert \bm g_{i,m}(\bm \phi_{i,m}^{t_0^\lambda})\right\Vert^2,
\label{EqLemma5Eq1}	
\end{aligned}
\end{equation}
where inequality in $(a)$ is since $\left\Vert \sum_{i=1}^Na_i\right\Vert ^2\leq N\sum_{i=1}^N\left\Vert a_i\right\Vert ^2$,
inequality in $(b)$ follows~\eqref{GlobalGradient}, inequality in $(c)$ is because of Lemma~\ref{Lemma3}.
\end{proof}

\begin{lemma}\label{Lemma6} 
\begin{equation}
\begin{aligned}
\mathbb{E}^{t_0^{\lambda}}\left\Vert \eta \sum_{t=t_0^{\lambda}}^{t_0^{\lambda+1}-1}\bm g_{(2)}^t\right\Vert^2 \leq 2\eta^2Q \left(4Q^2(1+Q)\rho^2\eta^2+Q\right)\frac{1}{K}\sum_{m=1}^{M}K_m\frac{1}{|\mathcal{A}^t_m|}\sum_{n\in \mathcal{A}^t_m}\left\Vert \bm g_{2,m,n}(\bm \phi_{i,m}^{t_0^{\lambda}})\right\Vert^2.
\label{EqLemma6}	
\end{aligned}
\end{equation}
\end{lemma}

\begin{proof}
Using the same idea of proofing Lemma~\ref{Lemma5}, we can prove Lemma~\ref{Lemma6}. Due to the page limitation, we omit the details.
\end{proof}

\subsection{Proof of Theorem~1}
We now prove Theorem~1.

\begin{proof}
Based on~\eqref{GlobalLossFunction} and~\eqref{EqAssumption1_1} in Assumption~1, we can obtain
\begin{equation}
\begin{aligned}
&\left\Vert\nabla F({\bm\theta_1})-\nabla F({\bm\theta_2})\right\Vert\\
&\overset{(a)}{=}\left\Vert\nabla \left(\frac{1}{K}\sum_{m=1}^{M}K_m {F_{m}({\bm\theta_1})}\right)-\nabla \left(\frac{1}{K}\sum_{m=1}^{M}K_m {F_{m}({\bm\theta_2})}\right)\right\Vert\\
&=\frac{1}{K}\sum_{m=1}^{M}K_m \left\Vert\nabla {F_{m}({\bm\theta_1})}-\nabla {F_{m}({\bm\theta_2})}\right\Vert\\
&\overset{(b)}{\leq}\rho \left\Vert{\bm\theta_1}-{\bm\theta_2}\right\Vert,
\label{Theorem1Proof1}
\end{aligned}
\end{equation}
where $(a)$ uses~\eqref{GlobalLossFunction}, and $(b)$ applies ~\eqref{EqAssumption1_1} in Assumption~1 and $K=\sum_{m=1}^M K_m$. The result shows that the global gradient follows \text{$\rho$-Lipschitz}. Based on this, during each global aggregation interval $[t_0, t_0+P-1]$, i.e., $[t_0, t_0^\Lambda-1]$, we have
\begin{equation}
\begin{aligned}
&F(\tilde{\bm\theta}^{t_0^\Lambda-1})-F(\tilde{\bm\theta}^{t_0})\\
&\leq \left\langle \nabla F(\tilde{\bm\theta}^{t_0}), \tilde{\bm\theta}^{t_0^\Lambda-1} -\tilde{\bm\theta}^{t_0}\right\rangle + \frac{\rho}{2}\left\Vert\tilde{\bm\theta}^{t_0^\Lambda-1} -\tilde{\bm\theta}^{t_0}\right\Vert^2\\
&\overset{(a)}\leq -\left\langle \nabla F(\tilde{\bm\theta}^{t_0}), \eta \sum_{t=t_0}^{t_0^\Lambda-1} \bm G^{t}\right\rangle + \frac{\rho}{2}\left\Vert\eta \sum_{t=t_0}^{t_0^\Lambda-1} \bm G^{t}\right\Vert^2\\
&= -\eta\sum_{t=t_0}^{t_0^\Lambda-1} \sum_{i=0}^2\left\langle \nabla F_{(i)}(\tilde{\bm\theta}^{t_0}), \bm G_{(i)}^{t}\right\rangle + \frac{\rho}{2}\left\Vert\eta \sum_{t=t_0}^{t_0^\Lambda-1} \bm G^{t}\right\Vert^2\\
&= -\eta\sum_{t=t_0}^{t_0^\Lambda-1} \sum_{i=0}^1\frac{1}{K}\sum_{m=1}^{M}K_m\left\langle \nabla F_{(i)}(\tilde{\bm\theta}^{t_0}), \bm g_{i,m}(\bm \phi_{i,m}^{t})\right\rangle\\
&\ \ -\eta\sum_{t=t_0}^{t_0^\Lambda-1}\frac{1}{K}\sum_{m=1}^{M}K_m\frac{1}{|\mathcal{A}^t_m|}\sum_{n\in \mathcal{A}^t_m}\left\langle \nabla F_{(2)}(\tilde{\bm\theta}^{t_0}), \bm g_{2,m,n}(\bm \phi_{2,m,n}^{t})\right\rangle + \frac{\rho}{2}\left\Vert\eta \sum_{t=t_0}^{t_0^\Lambda-1} \bm G^{t}\right\Vert^2\\
&= \eta\sum_{t=t_0}^{t_0^\Lambda-1} \sum_{i=0}^1\frac{1}{K}\sum_{m=1}^{M}K_m\left\langle -\nabla F_{(i)}(\tilde{\bm\theta}^{t_0}), \bm g_{i,m}(\bm \phi_{i,m}^{t})-\bm g_{i,m}(\bm \phi_{i,m}^{t_0})\right\rangle\\
&\ \ -\eta\sum_{t=t_0}^{t_0^\Lambda-1} \sum_{i=0}^1\frac{1}{K}\sum_{m=1}^{M}K_m\left\langle \nabla F_{(i)}(\tilde{\bm\theta}^{t_0}), \bm g_{i,m}(\bm \phi_{i,m}^{t_0})\right\rangle\\
&\ \ +\eta\sum_{t=t_0}^{t_0^\Lambda-1}\frac{1}{K}\sum_{m=1}^{M}K_m\frac{1}{|\mathcal{A}^t_m|}\sum_{n\in \mathcal{A}^t_m}\left\langle - \nabla F_{(2)}(\tilde{\bm\theta}^{t_0}), \bm g_{2,m,n}(\bm \phi_{2,m,n}^{t}) -\bm g_{2,m,n}(\bm \phi_{2,m,n}^{t_0})\right\rangle \\
&\ \ -\eta\sum_{t=t_0}^{t_0^\Lambda-1}\frac{1}{K}\sum_{m=1}^{M}K_m\frac{1}{|\mathcal{A}^t_m|}\sum_{n\in \mathcal{A}^t_m}\left\langle \nabla F_{(2)}(\tilde{\bm\theta}^{t_0}), \bm g_{2,m,n}(\bm \phi_{2,m,n}^{t_0})\right\rangle  + \frac{\rho}{2}\left\Vert\eta \sum_{t=t_0}^{t_0^\Lambda-1} \bm G^{t}\right\Vert^2\\
&\overset{(b)}\leq \frac{\eta}{2}\sum_{t=t_0}^{t_0^\Lambda-1} \sum_{i=0}^1\frac{1}{K}\sum_{m=1}^{M}K_m\left[ \left\Vert\nabla F_{(i)}(\tilde{\bm\theta}^{t_0})\right\Vert^2+ \left\Vert \bm g_{i,m}(\bm \phi_{i,m}^{t})-\bm g_{i,m}(\bm \phi_{i,m}^{t_0})\right\Vert^2\right]\\
&\ \ -\eta\sum_{t=t_0}^{t_0^\Lambda-1} \sum_{i=0}^1\frac{1}{K}\sum_{m=1}^{M}K_m\left\langle \nabla F_{(i)}(\tilde{\bm\theta}^{t_0}), \bm g_{i,m}(\bm \phi_{i,m}^{t_0})\right\rangle\\
&\ \ +\frac{\eta}{2}\sum_{t=t_0}^{t_0^\Lambda-1} \frac{1}{K}\sum_{m=1}^{M}K_m\frac{1}{|\mathcal{A}^t_m|}\sum_{n\in \mathcal{A}^t_m}\left[ \left\Vert\nabla F_{(2)}(\tilde{\bm\theta}^{t_0})\right\Vert^2+ \left\Vert \bm g_{2,m,n}(\bm \phi_{2,m,n}^{t})-\bm g_{2,m,n}(\bm \phi_{2,m,n}^{t_0})\right\Vert^2\right]\\
&\ \ -\eta\sum_{t=t_0}^{t_0^\Lambda-1}\frac{1}{K}\sum_{m=1}^{M}K_m\frac{1}{|\mathcal{A}^t_m|}\sum_{n\in \mathcal{A}^t_m}\left\langle \nabla F_{(2)}(\tilde{\bm\theta}^{t_0}), \bm g_{2,m,n}(\bm \phi_{2,m,n}^{t_0})\right\rangle + \frac{\rho}{2}\left\Vert\eta \sum_{t=t_0}^{t_0^\Lambda-1} \bm G^{t}\right\Vert^2,
\label{Theorem1Proof2}
\end{aligned}
\end{equation}
where inequality $(a)$ comes from~\eqref{Appendix_GlobalUpdate} and $(b)$ applies $A\cdot B \leq \frac{A^2+B^2}{2}$. Taking conditional expectation at $t_0$ on~\eqref{Theorem1Proof2}, we obtain
\begin{align}
&\mathbb{E}^{t_0}[F(\tilde{\bm\theta}^{t_0^\Lambda-1})]-F(\tilde{\bm\theta}^{t_0})\notag\\
&\leq \frac{\eta}{2}\sum_{t=t_0}^{t_0^\Lambda-1} \sum_{i=0}^1\frac{1}{K}\sum_{m=1}^{M}K_m\left[ \left\Vert\nabla F_{(i)}(\tilde{\bm\theta}^{t_0})\right\Vert^2+ \mathbb{E}^{t_0}\left\Vert \bm g_{i,m}(\bm \phi_{i,m}^{t})-\bm g_{i,m}(\bm \phi_{i,m}^{t_0})\right\Vert^2\right]\notag\\
&\ \ -\eta\sum_{t=t_0}^{t_0^\Lambda-1} \sum_{i=0}^1\frac{1}{K}\sum_{m=1}^{M}K_m\left\langle \nabla F_{(i)}(\tilde{\bm\theta}^{t_0}), \bm g_{i,m}(\bm \phi_{i,m}^{t_0})\right\rangle\notag\\
&\ \ +\frac{\eta}{2}\sum_{t=t_0}^{t_0^\Lambda-1} \frac{1}{K}\sum_{m=1}^{M}K_m\frac{1}{|\mathcal{A}^t_m|}\sum_{n\in \mathcal{A}^t_m}\left[ \left\Vert\nabla F_{(2)}(\tilde{\bm\theta}^{t_0})\right\Vert^2+ \mathbb{E}^{t_0}\left\Vert \bm g_{2,m,n}(\bm \phi_{2,m,n}^{t})-\bm g_{2,m,n}(\bm \phi_{2,m,n}^{t_0})\right\Vert^2\right]\notag\\
&\ \ -\eta\sum_{t=t_0}^{t_0^\Lambda-1}\frac{1}{K}\sum_{m=1}^{M}K_m\frac{1}{|\mathcal{A}^t_m|}\sum_{n\in \mathcal{A}^t_m}\left\langle \nabla F_{(2)}(\tilde{\bm\theta}^{t_0}), \bm g_{2,m,n}(\bm \phi_{2,m,n}^{t_0})\right\rangle + \frac{\rho}{2}\mathbb{E}^{t_0}\left\Vert\eta \sum_{t=t_0}^{t_0^\Lambda-1} \bm G^{t}\right\Vert^2\notag\\
&\overset{(a)}\leq \frac{\eta}{2}\sum_{t=t_0}^{t_0^\Lambda-1} \sum_{i=0}^1\frac{1}{K}\sum_{m=1}^{M}K_m\left\Vert\nabla F_{(i)}(\tilde{\bm\theta}^{t_0})\right\Vert^2+ \frac{\eta}{2}\sum_{t=t_0}^{t_0^\Lambda-1} \sum_{i=0}^1\frac{1}{K}\sum_{m=1}^{M}K_m\mathbb{E}^{t_0}\left\Vert \bm g_{i,m}(\bm \phi_{i,m}^{t})-\bm g_{i,m}(\bm \phi_{i,m}^{t_0})\right\Vert^2\notag\\
&\ \ -\eta\sum_{t=t_0}^{t_0^\Lambda-1} \sum_{i=0}^1\frac{1}{K}\sum_{m=1}^{M}K_m\left\Vert\nabla F_{(i)}(\tilde{\bm\theta}^{t_0})\right\Vert^2\notag\\
&\ \ +\frac{\eta}{2}\sum_{t=t_0}^{t_0^\Lambda-1} \frac{1}{K}\sum_{m=1}^{M}K_m\frac{1}{|\mathcal{A}^t_m|}\sum_{n\in \mathcal{A}^t_m} \left\Vert\nabla F_{(2)}(\tilde{\bm\theta}^{t_0})\right\Vert^2+ \frac{\eta}{2}\sum_{t=t_0}^{t_0^\Lambda-1} \frac{1}{K}\sum_{m=1}^{M}K_m\frac{1}{|\mathcal{A}^t_m|}\sum_{n\in \mathcal{A}^t_m}\!\mathbb{E}^{t_0}\left\Vert \bm g_{2,m,n}(\bm \phi_{2,m,n}^{t})-\bm g_{2,m,n}(\bm \phi_{2,m,n}^{t_0})\right\Vert^2\notag\\
&\ \ -\eta\sum_{t=t_0}^{t_0^\Lambda-1}\frac{1}{K}\sum_{m=1}^{M}K_m\frac{1}{|\mathcal{A}^t_m|}\sum_{n\in \mathcal{A}^t_m}\left\Vert\nabla F_{(2)}(\tilde{\bm\theta}^{t_0})\right\Vert^2 + \frac{\rho}{2}\mathbb{E}^{t_0}\left\Vert\eta \sum_{t=t_0}^{t_0^\Lambda-1} \bm G^{t}\right\Vert^2\notag\\
&\leq -\frac{\eta}{2}\sum_{t=t_0}^{t_0^\Lambda-1} \sum_{i=0}^1\frac{1}{K}\sum_{m=1}^{M}K_m\left\Vert\nabla F_{(i)}(\tilde{\bm\theta}^{t_0})\right\Vert^2+ \frac{\eta}{2}\sum_{t=t_0}^{t_0^\Lambda-1} \sum_{i=0}^1\frac{1}{K}\sum_{m=1}^{M}K_m\mathbb{E}^{t_0}\left\Vert \bm g_{i,m}(\bm \phi_{i,m}^{t})-\bm g_{i,m}(\bm \phi_{i,m}^{t_0})\right\Vert^2\notag\\
&\ \ -\frac{\eta}{2}\sum_{t=t_0}^{t_0^\Lambda-1} \frac{1}{K}\sum_{m=1}^{M}K_m\frac{1}{|\mathcal{A}^t_m|}\sum_{n\in \mathcal{A}^t_m} \left\Vert\nabla F_{(2)}(\tilde{\bm\theta}^{t_0})\right\Vert^2+ \frac{\eta}{2}\sum_{t=t_0}^{t_0^\Lambda-1} \frac{1}{K}\sum_{m=1}^{M}K_m\frac{1}{|\mathcal{A}^t_m|}\!\sum_{n\in \mathcal{A}^t_m}\!\mathbb{E}^{t_0}\!\left\Vert \bm g_{2,m,n}(\bm \phi_{2,m,n}^{t})\!-\!\bm g_{2,m,n}(\bm \phi_{2,m,n}^{t_0})\right\Vert^2\notag\\
&\ \  + \frac{\rho}{2}\mathbb{E}^{t_0}\left\Vert\eta \sum_{t=t_0}^{t_0^\Lambda-1} \bm G^{t}\right\Vert^2\notag\\
&\overset{(b)}\leq -\frac{\eta}{2}\sum_{\lambda=0}^{\Lambda-1}\sum_{t=t_0^{\lambda}}^{t_0^{\lambda+1}-1} \sum_{i=0}^1\frac{1}{K}\sum_{m=1}^{M}K_m\left\Vert\nabla F_{(i)}(\tilde{\bm\theta}^{t_0})\right\Vert^2+ \frac{\eta}{2}\sum_{\lambda=0}^{\Lambda-1}\sum_{t=t_0^{\lambda}}^{t_0^{\lambda+1}-1}\sum_{i=0}^1\frac{1}{K}\sum_{m=1}^{M}K_m\mathbb{E}^{t_0}\left\Vert \bm g_{i,m}(\bm \phi_{i,m}^{t})-\bm g_{i,m}(\bm \phi_{i,m}^{t_0})\right\Vert^2\notag\\
&\ \ -\frac{\eta}{2}\sum_{\lambda=0}^{\Lambda-1}\sum_{t=t_0^{\lambda}}^{t_0^{\lambda+1}-1} \frac{1}{K}\sum_{m=1}^{M}K_m\frac{1}{|\mathcal{A}^t_m|}\sum_{n\in \mathcal{A}^t_m} \left\Vert\nabla F_{(2)}(\tilde{\bm\theta}^{t_0})\right\Vert^2\notag\\
&\ \ + \frac{\eta}{2}\sum_{\lambda=0}^{\Lambda-1}\sum_{t=t_0^{\lambda}}^{t_0^{\lambda+1}-1} \frac{1}{K}\sum_{m=1}^{M}K_m\frac{1}{|\mathcal{A}^t_m|}\!\sum_{n\in \mathcal{A}^t_m}\!\mathbb{E}^{t_0}\!\left\Vert \bm g_{2,m,n}(\bm \phi_{2,m,n}^{t})\!-\!\bm g_{2,m,n}(\bm \phi_{2,m,n}^{t_0})\right\Vert^2 + \frac{\rho}{2}\mathbb{E}^{t_0}\left\Vert\eta \sum_{\lambda=0}^{\Lambda-1}\sum_{t=t_0^{\lambda}}^{t_0^{\lambda+1}-1} \bm G^{t}\right\Vert^2\notag\\
&\leq -\frac{\eta}{2}\sum_{\lambda=0}^{\Lambda-1}\sum_{t=t_0^{\lambda}}^{t_0^{\lambda+1}-1} \sum_{i=0}^1\frac{1}{K}\sum_{m=1}^{M}K_m\left\Vert\nabla F_{(i)}(\tilde{\bm\theta}^{t_0})\right\Vert^2+ \frac{\eta}{2}\sum_{\lambda=0}^{\Lambda-1}\sum_{t=t_0^{\lambda}}^{t_0^{\lambda+1}-1}\sum_{i=0}^1\frac{1}{K}\sum_{m=1}^{M}K_m\mathbb{E}^{t_0}\left\Vert \bm g_{i,m}(\bm \phi_{i,m}^{t})-\bm g_{i,m}(\bm \phi_{i,m}^{t_0})\right\Vert^2\notag\\
&\ \ -\frac{\eta}{2}\sum_{\lambda=0}^{\Lambda-1}\sum_{t=t_0^{\lambda}}^{t_0^{\lambda+1}-1} \frac{1}{K}\sum_{m=1}^{M}K_m\frac{1}{|\mathcal{A}^t_m|}\sum_{n\in \mathcal{A}^t_m} \left\Vert\nabla F_{(2)}(\tilde{\bm\theta}^{t_0})\right\Vert^2 \notag\\
&\ \ +\frac{\eta}{2}\sum_{\lambda=0}^{\Lambda-1}\sum_{t=t_0^{\lambda}}^{t_0^{\lambda+1}-1} \frac{1}{K}\sum_{m=1}^{M}K_m\frac{1}{|\mathcal{A}^t_m|}\!\sum_{n\in \mathcal{A}^t_m}\!\mathbb{E}^{t_0}\!\left\Vert \bm g_{2,m,n}(\bm \phi_{2,m,n}^{t})\!-\!\bm g_{2,m,n}(\bm \phi_{2,m,n}^{t_0})\right\Vert^2 + \frac{\rho\Lambda}{2}\sum_{\lambda=0}^{\Lambda-1}\sum_{i=0}^2\mathbb{E}^{t_0}\left\Vert\eta \sum_{t=t_0^{\lambda}}^{t_0^{\lambda+1}-1} \bm G_{(i)}^{t}\right\Vert^2,
\label{Theorem1Proof3}
\end{align}
where $(a)$ is because of~\eqref{EqAssumption2_1} in Assumption~2 and $(b)$ applies $\sum_{t=t_0}^{t_0^\Lambda-1}=\sum_{\lambda=0}^{\Lambda-1}\sum_{t=t_0^{\lambda}}^{t_0^{\lambda+1}-1}$.

Applying Lemmas~\ref{Lemma1}--\ref{Lemma6} into~\eqref{Theorem1Proof3}, we have 
\begin{align}
&\mathbb{E}^{t_0}[F(\tilde{\bm\theta}^{t_0^\Lambda-1})]-F(\tilde{\bm\theta}^{t_0})\notag\\
&\overset{(a)}\leq -\frac{\eta}{2}\sum_{\lambda=0}^{\Lambda-1}\sum_{t=t_0^{\lambda}}^{t_0^{\lambda+1}-1} \sum_{i=0}^1\frac{1}{K}\sum_{m=1}^{M}K_m\left\Vert\nabla F_{(i)}(\tilde{\bm\theta}^{t_0})\right\Vert^2+ 2Q^2\left(1+Q\right)\rho^2\eta^3 \sum_{\lambda=0}^{\Lambda-1}\sum_{i=0}^1\frac{1}{K}\sum_{m=1}^{M}K_m\left\Vert \bm g_{i,m}(\bm \phi_{i,m}^{t_0})\right\Vert^2\notag\\
&\ \ -\frac{\eta}{2}\sum_{\lambda=0}^{\Lambda-1}\sum_{t=t_0^{\lambda}}^{t_0^{\lambda+1}-1} \frac{1}{K}\sum_{m=1}^{M}K_m\frac{1}{|\mathcal{A}^t_m|}\sum_{n\in \mathcal{A}^t_m} \left\Vert\nabla F_{(2)}(\tilde{\bm\theta}^{t_0})\right\Vert^2+ 2Q^2(1+Q)\rho^2\eta^3 \sum_{\lambda=0}^{\Lambda-1}\frac{1}{K}\sum_{m=1}^{M}K_m\frac{1}{|\mathcal{A}^t_m|}\!\sum_{n\in \mathcal{A}^t_m}\!\!\left\Vert \bm g_{2,m,n}(\bm \phi_{2,m,n}^{t_0})\right\Vert^2\notag\\
&\ \ +\rho\Lambda\eta^2 Q\left(4Q^2\left(1+Q\right)\rho^2\eta^2+Q\right)\sum_{\lambda=0}^{\Lambda-1}\frac{1}{K}\sum_{m=1}^{M}K_m \left\Vert \bm g_{i,m}(\bm \phi_{i,m}^{t_0})\right\Vert^2\notag\\
&\ \ +\rho\Lambda\eta^2Q\left(4Q^2\left(1+Q\right)\rho^2\eta^2+Q\right)\sum_{\lambda=0}^{\Lambda-1}\frac{1}{K}\sum_{m=1}^{M}K_m\frac{1}{|\mathcal{A}^t_m|}\sum_{n\in \mathcal{A}^t_m} \left\Vert \bm g_{2,m,n}(\bm \phi_{i,m}^{t_0})\right\Vert^2\notag\\
&\leq -\frac{\eta}{2}\sum_{\lambda=0}^{\Lambda-1}\sum_{t=t_0^{\lambda}}^{t_0^{\lambda+1}-1} \sum_{i=0}^1\frac{1}{K}\sum_{m=1}^{M}K_m\left\Vert\nabla F_{(i)}(\tilde{\bm\theta}^{t_0})\right\Vert^2 -\frac{\eta}{2}\sum_{\lambda=0}^{\Lambda-1}\sum_{t=t_0^{\lambda}}^{t_0^{\lambda+1}-1} \frac{1}{K}\sum_{m=1}^{M}K_m\frac{1}{|\mathcal{A}^t_m|}\sum_{n\in \mathcal{A}^t_m} \left\Vert\nabla F_{(2)}(\tilde{\bm\theta}^{t_0})\right\Vert^2\notag\\
&\ \ + \left(\rho\Lambda\eta^2Q^2+2Q^2\left(2\rho\Lambda\eta^2Q+\eta\right)\left(1+Q\right)\rho^2\eta^2\right) \sum_{\lambda=0}^{\Lambda-1}\sum_{i=0}^1\frac{1}{K}\sum_{m=1}^{M}K_m\left\Vert \bm g_{i,m}(\bm \phi_{i,m}^{t_0})\right\Vert^2\notag\\
&\ \ + \left(\rho\Lambda\eta^2Q^2+2Q^2\left(2\rho\Lambda\eta^2Q+\eta\right)\left(1+Q\right)\rho^2\eta^2\right)\sum_{\lambda=0}^{\Lambda-1} \frac{1}{K}\sum_{m=1}^{M}K_m\frac{1}{|\mathcal{A}^t_m|}\!\sum_{n\in \mathcal{A}^t_m}\!\!\left\Vert \bm g_{2,m,n}(\bm \phi_{2,m,n}^{t_0})\right\Vert^2\notag\\
&\overset{(b)}\leq -\frac{\eta}{2}\sum_{\lambda=0}^{\Lambda-1}\sum_{t=t_0^{\lambda}}^{t_0^{\lambda+1}-1} \sum_{i=0}^1\frac{1}{K}\sum_{m=1}^{M}K_m\left\Vert\nabla F_{(i)}(\tilde{\bm\theta}^{t_0})\right\Vert^2 -\frac{\eta}{2}\sum_{\lambda=0}^{\Lambda-1}\sum_{t=t_0^{\lambda}}^{t_0^{\lambda+1}-1} \frac{1}{K}\sum_{m=1}^{M}K_m\frac{1}{|\mathcal{A}^t_m|}\sum_{n\in \mathcal{A}^t_m} \left\Vert\nabla F_{(2)}(\tilde{\bm\theta}^{t_0})\right\Vert^2\notag\\
&\ \ + \left(\rho\Lambda\eta^2Q^2+2Q^2\left(2\rho\Lambda\eta^2Q+\eta\right)\left(1+Q\right)\rho^2\eta^2\right) \sum_{\lambda=0}^{\Lambda-1}\sum_{i=0}^1\frac{1}{K}\sum_{m=1}^{M}K_m\left(\delta^2+\left\Vert \nabla_{(i)}F({\tilde{\bm\theta}^{t_0}})\right\Vert^2\right)\notag\\
&\ \ + \left(\rho\Lambda\eta^2Q^2+2Q^2\left(2\rho\Lambda\eta^2Q+\eta\right)\left(1+Q\right)\rho^2\eta^2\right)\sum_{\lambda=0}^{\Lambda-1} \frac{1}{K}\sum_{m=1}^{M}K_m\frac{1}{|\mathcal{A}^t_m|}\!\sum_{n\in \mathcal{A}^t_m}\left(\delta^2+\left\Vert \nabla_{(2)}F({\tilde{\bm\theta}^{t_0}})\right\Vert^2\right)\notag\\
&\overset{(c)}\leq -\frac{\eta P}{2} \sum_{i=0}^1\frac{1}{K}\sum_{m=1}^{M}K_m\left\Vert\nabla F_{(i)}(\tilde{\bm\theta}^{t_0})\right\Vert^2 -\frac{\eta P}{2} \frac{1}{K}\sum_{m=1}^{M}K_m\frac{1}{|\mathcal{A}^t_m|}\sum_{n\in \mathcal{A}^t_m} \left\Vert\nabla F_{(2)}(\tilde{\bm\theta}^{t_0})\right\Vert^2\notag\\
&\ \ + \frac{\eta P}{2} \left(2\rho\Lambda\eta Q+4Q\left(2\rho\Lambda\eta Q+1\right)\left(1+Q\right)\rho^2\eta^2\right)\sum_{i=0}^1\frac{1}{K}\sum_{m=1}^{M}K_m\left(\delta^2+\left\Vert \nabla_{(i)}F({\tilde{\bm\theta}^{t_0}})\right\Vert^2\right)\notag\\
&\ \ + \frac{\eta P}{2} \left(2\rho\Lambda\eta Q+4Q\left(2\rho\Lambda\eta Q+1\right)\left(1+Q\right)\rho^2\eta^2\right)\frac{1}{K}\sum_{m=1}^{M}K_m\frac{1}{|\mathcal{A}^t_m|}\!\sum_{n\in \mathcal{A}^t_m}\left(\delta^2+\left\Vert \nabla_{(2)}F({\tilde{\bm\theta}^{t_0}})\right\Vert^2\right)\notag\\
&\leq -\frac{\eta P}{2}\left(1- 2\rho\Lambda\eta Q-4Q\left(2\rho\Lambda\eta Q+1\right)\left(1+Q\right)\rho^2\eta^2\right)\left\Vert\nabla F(\tilde{\bm\theta}^{t_0^\Lambda-1})\right\Vert^2 + \frac{3\eta P}{2}\left(2\rho\Lambda\eta Q+4Q\left(2\rho\Lambda\eta Q+1\right)\left(1+Q\right)\rho^2\eta^2\right) \delta^2\notag\\
&\leq -\frac{\eta P}{2}\left(1-2P\rho\eta - 8Q^2\rho^2\eta^2 - 16Q^2P\rho^3\eta^3\right) \left\Vert\nabla F(\tilde{\bm\theta}^{t_0^\Lambda-1})\right\Vert^2+ \frac{3\eta P}{2}\left(2P\rho\eta + 8Q^2\rho^2\eta^2 + 16Q^2P\rho^3\eta^2\eta\right) \delta^2\notag\\
&\overset{(d)}\leq -\frac{\eta P}{2}\left(1-2P\rho\eta - 8Q^2\rho^2\eta^2 - 16Q^2P\rho^3\eta^3\right) \left\Vert\nabla F(\tilde{\bm\theta}^{t_0^\Lambda-1})\right\Vert^2+ \frac{3\eta P}{2}\left(2P\rho\eta + 8Q^2\rho^2\eta^2 + 8Q^2\rho^2\eta^2\right) \delta^2\notag\\
&\leq -\frac{\eta P}{2}\left(1-2P\rho\eta - 8Q^2\rho^2\eta^2 - 16Q^2P\rho^3\eta^3\right) \left\Vert\nabla F(\tilde{\bm\theta}^{t_0^\Lambda-1})\right\Vert^2+ \frac{3\eta P}{2}\left(2P\rho\eta + 16Q^2\rho^2\eta^2\right) \delta^2.
\label{Theorem1Proof4}
\end{align}
Here $(a)$ follows Lemmas~\ref{Lemma3}-\ref{Lemma6}, $(b)$ applies Lemmas~\ref{Lemma1}-\ref{Lemma2} and $\bm \phi_{0,m}^{t} = \bm \phi_{1,m}^{t}=\bm\phi_{2,m,n}^{t} = \tilde{\bm\theta}^{t}$ when $t=t_0$, $(c)$ is because $\sum_{\lambda=0}^{\Lambda-1}\sum_{t=t_0^{\lambda}}^{t_0^{\lambda+1}-1} 1 = \Lambda Q = P$, and $(d)$ is on the setting $\eta\leq \frac{1}{2P\rho}$. 

Choosing $\eta\leq \frac{1}{8P\rho}$, we can further bound the R.H.S. expression of~\eqref{Theorem1Proof4} as follow
\begin{equation}
\begin{aligned}
&\mathbb{E}^{t_0}[F(\tilde{\bm\theta}^{t_0^\Lambda-1})]-F(\tilde{\bm\theta}^{t_0})\\
&\leq -\frac{\eta P}{2}\left(1-\frac{1}{4}-\frac{1}{8\Lambda^2}-\frac{1}{32\Lambda^2}\right) \left\Vert\nabla F(\tilde{\bm\theta}^{t_0})\right\Vert^2+ \frac{3\eta P}{2}\left(2P\rho\eta + 16Q^2\rho^2\eta^2\right) \delta^2\\
&\overset{(a)}\leq -\frac{\eta P}{4}\left\Vert\nabla F(\tilde{\bm\theta}^{t_0})\right\Vert^2+ \frac{3\eta P}{2}\left(2P\rho\eta + 16Q^2\rho^2\eta^2\right) \delta^2,\\
&\leq -\frac{\eta P}{4}\left\Vert\nabla F(\tilde{\bm\theta}^{t_0})\right\Vert^2+ 3P^2\rho\eta^2\delta^2 + 24Q^2P\rho^2\eta^3\delta^2,
\label{Theorem1Proof5}
\end{aligned}
\end{equation}
where $(a)$ comes from $\Lambda \geq 1$.

Rearranging~\eqref{Theorem1Proof5}, we have

\begin{equation}
\begin{aligned}
&\left\Vert\nabla F(\tilde{\bm\theta}^{t_0})\right\Vert^2 \leq \frac{4\left(F(\tilde{\bm\theta}^{t_0})-\mathbb{E}^{t_0}[F(\tilde{\bm\theta}^{t_0^\Lambda-1})]\right)}{\eta P} + 12P\rho\eta\delta^2 + 96Q^2\rho^2\eta^2 \delta^2.
\label{Theorem1Proof6}
\end{aligned}
\end{equation}

Taking the total expectation and averaging all global aggregation intervals, i.e., $t_0 = 0, P, 2P, \cdots, (R-1)P$ and $T=RP$, we obtain
\begin{equation}
\begin{aligned}
&\mathbb{E}\left[\frac{1}{R}\sum_{r=0}^{R-1} \left\Vert\nabla F(\tilde{\bm\theta}^{rP})\right\Vert^2\right]\\
&\leq \mathbb{E}\left[\frac{1}{R}\sum_{r=0}^{R-1}\frac{4\left(F(\tilde{\bm\theta}^{rP})-\mathbb{E}^{t_0}[F(\tilde{\bm\theta}^{(r+1)P-1})]\right)}{\eta P}\right]+ \mathbb{E}\left[\frac{1}{R}\sum_{r=0}^{R-1}\left(12P\rho\eta\delta^2 + 96Q^2\rho^2\eta^2 \delta^2\right)\right]\\
&\leq \frac{4\left(F(\tilde{\bm\theta}^{0})-\mathbb{E}[F(\tilde{\bm\theta}^{T})]\right)}{\eta T} + 12P\rho\eta\delta^2 + 96Q^2\rho^2\eta^2 \delta^2.
\label{Theorem1Proof7}
\end{aligned}
\end{equation}

\end{proof}

\section{}
\label{Appendix_B}
The CNN model used in the proposed HSGD algorithm and all baselines is shown in Fig.~\ref{CNN_model_1}.

\begin{figure*}[h]
		\centering
		\includegraphics[width=1\linewidth]{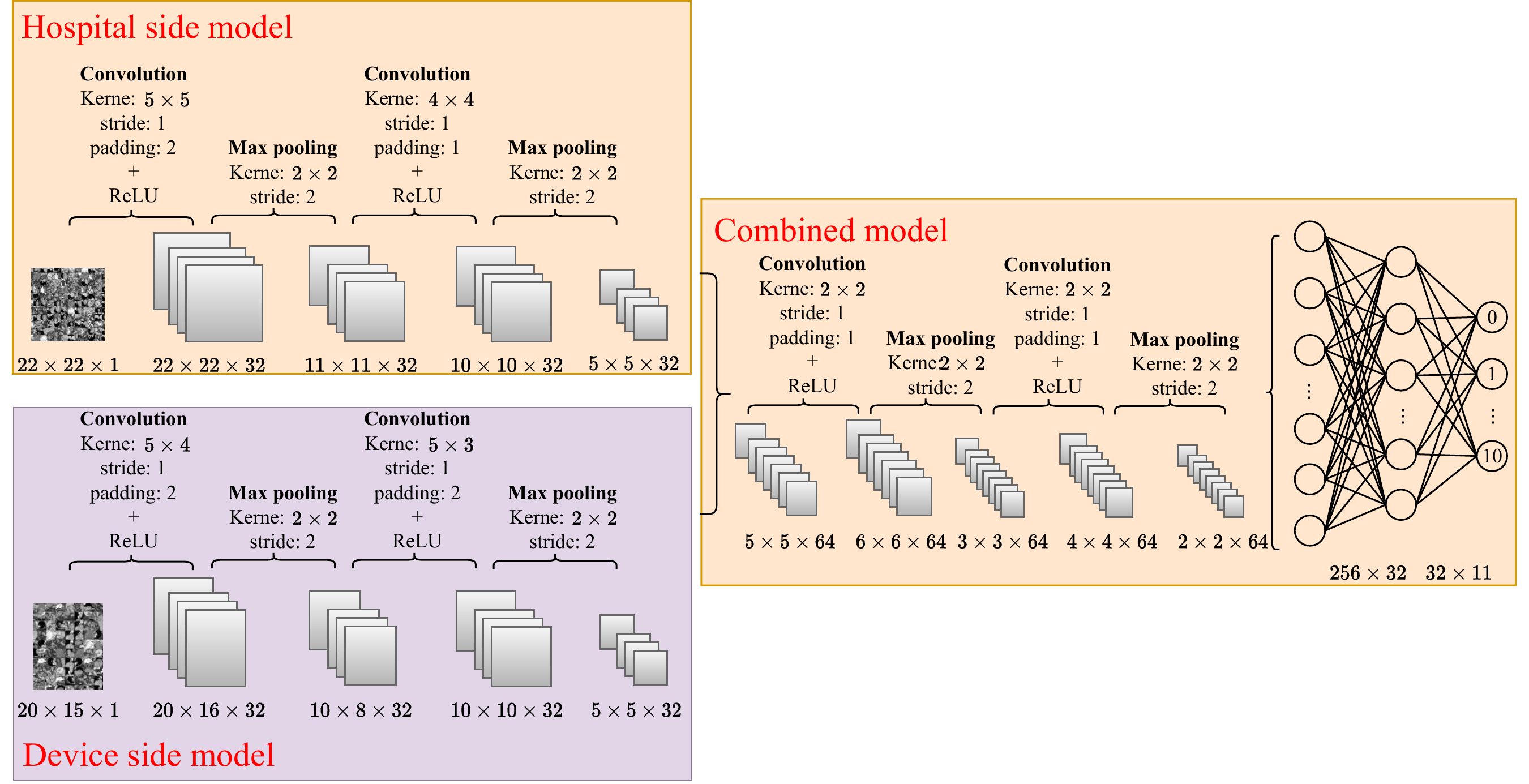} 
		\captionsetup{font={scriptsize}}
		\caption{The CNN model used in the proposed HSGD algorithm and baselines. The entire model consists of the hospital side model, the device side model, and the combined model.}
		\label{CNN_model_1}
\end{figure*}

\end{document}